\documentclass[10pt]{article}

\usepackage{microtype}
\usepackage{graphicx}
\usepackage{subfigure}
\usepackage{booktabs} 
\usepackage{fullpage}

\usepackage[breaklinks=true, colorlinks=true]{hyperref}

\usepackage{amssymb}

\usepackage{amsmath}
\usepackage{amssymb}
\usepackage{mathtools}
\usepackage{amsthm}

\usepackage[capitalize,noabbrev]{cleveref}

\theoremstyle{plain}
\newtheorem{theorem}{Theorem}[section]

\newtheorem{lemma}[theorem]{Lemma}
\newtheorem{corollary}[theorem]{Corollary}
\theoremstyle{definition}

\theoremstyle{remark}
\newtheorem{remark}[theorem]{Remark}

\usepackage[textsize=tiny]{todonotes}
\usepackage{natbib}
\usepackage{amsmath,amsfonts}
\usepackage{amsthm}
\usepackage{csquotes}
\usepackage{url}
\usepackage{bm}

\usepackage[ruled,vlined,linesnumbered, algo2e]{algorithm2e}

\newcommand{\Dim}{{\mathcal Dim}}

\newcommand{\rE}{{\mathbb E}}

\newcommand{\rR}{{\mathbb R}}

\newcommand{\tr}{{\mathrm{trace}}}
\newcommand{\rP}{{\mathbb{P}}}

\newcommand{\cQ}{{\mathcal Q}}
\newcommand{\cC}{{\mathcal C}}
\newcommand{\cD}{{\mathcal D}}

\newcommand{\cY}{{\mathcal Y}}

\newcommand{\cX}{{\mathcal X}}

\newcommand{\cF}{{\mathcal F}}
\newcommand{\cP}{{\mathcal P}}
\newcommand{\cT}{{\mathcal T}}
\newcommand{\cL}{{\mathcal L}}

\newcommand{\hy}{\hat{y}}

\newcommand{\hh}{\widehat{h}}
\newcommand{\hf}{\widehat{f}}
\newcommand{\hw}{\widehat{\bf w}}

\newcommand{\sign}{\mathrm{sgn}}

\renewcommand{\tr}{\mathrm{tr}}

\usepackage{mathtools}

\DeclarePairedDelimiter\ang{\langle}{\rangle}

\newcommand{\x}{\ensuremath{\mathbf{x}}}
\newcommand{\g}{\ensuremath{\mathbf{g}}}

\newcommand{\E}[1]{{\mathbb{E}\left[{#1}\right]}}

\newcommand{\w}{\ensuremath{\mathbf{w}}}
\newcommand{\z}{\ensuremath{\mathbf{z}}}

\newcommand{\ind}[1]{\ 1\hspace{-2.3mm}{1}{\left\{#1\right\}}}

\newcommand{\argmax}[1]{\underset{#1}{\mathrm{argmax}} \:}
\newcommand{\argmin}[1]{\underset{#1}{\mathrm{argmin}} \:}
\renewcommand{\max}[1]{\underset{#1}{\mathrm{max}} \:}

\newcommand{\ep}{\mathcal{E}}

\begin{document}

\title{\bf Fast Rates in Pool-Based Batch Active Learning\thanks{Part of this work appeared in ICML 2022.}}






\author{
Claudio Gentile\thanks{Google Research, New York}
\and Zhilei Wang\thanks{New York University, New York}
\and 
Tong Zhang\thanks{Google Research, New York, and The Hong Kong University of Science and Technology, Hong Kong}
}



\vskip 0.3in




\date{}
\maketitle

\begin{abstract}
We consider a batch active learning scenario where the learner adaptively issues batches of points to a labeling oracle. Sampling labels in batches is highly desirable in practice due to the smaller number of interactive rounds with the labeling oracle (often human beings). However,  batch active learning typically pays the price of a reduced adaptivity, leading to suboptimal results. In this paper we propose a solution which requires a careful trade off between the informativeness of the queried points and their diversity. 
We theoretically investigate batch active learning in the practically relevant scenario where the unlabeled pool of data is available beforehand ({\em pool-based} active learning). We analyze a novel stage-wise greedy algorithm and show that, as a function of the label complexity, the excess risk of this algorithm matches the known minimax rates in standard statistical learning settings. Our results also exhibit a mild dependence on the batch size. These are the first theoretical results that employ careful trade offs between informativeness and diversity to rigorously quantify the statistical performance of batch active learning in the pool-based scenario.
\end{abstract}

\section{Introduction}
\label{s:intro}
The aim of Active Learning is to reduce the data requirement of training processes through the careful selection of informative subsets of the data across several interactive rounds. 
This increased interactive power enables the adaptation of the sampling process to the actual state of the learning algorithm at hand, yet this benefit comes at the price of frequent re-training of the model and increased interactions with the labeling oracle (which is often just a pool of human labelers). 

The {\em batch} mode of active learning is one where labels are queried in batches of suitable size, and the models are re-trained/updated either after each batch or even less frequently. This sampling mode often corresponds to the way labels are gathered in practical large-scale processing pipelines. 

Batch active learning tries to strike a reasonable balance between the benefits of adaptivity and the costs associated with interaction and re-training. Yet, since the sampling is split into batches, and model updates can only be performed at the end of each batch, a batch active learning algorithm has to prevent to the extent possible the sampling of redundant points. 
The standard trade-off that arises is then to ensure that the sampled points are {\em informative} enough for the model, if taken in isolation, while at the same time being {\em diverse} enough so as to avoid sampling redundant labels.

We study batch active learning in the {\em pool-based} model, where an unlabeled pool of data is made available to the algorithm beforehand, and the goal is to single out a subset of the data so as to achieve the same statistical performance as if training were carried out on the entire pool. 
In this setting, we describe and analyze novel algorithms that obtain
minimax rates of convergence of their excess risk as a function of the number of requested labels. Interestingly enough, these optimal rates are retained even if we allow the batch size to grow with the pool size, the actual trade-off being ruled by the amount of noise in the data.
Another appealing aspect is that our algorithms guarantee a number of re-training rounds which is at worst logarithmic, while being able to automatically adapt to the level of noise. 

We operate in specific realizable settings, starting with linear or generalized linear models, and then extending our results to the more general non-linear setting. Unlike what is traditionally done by many algorithmic solutions to active learning available in the literature (e.g., \cite{marginpaper,bl13,Zhang2021}), we do not formulate strong assumptions on the input distribution. We establish careful trade-offs between the informativeness and the diversity of the queried labels, and rigorously quantify the statistical performance on batch active learning in a noisy pool-based setting. To our knowledge, these are the first guarantees of this kind that apply to a noisy (hence realistic) batch pool-based active learning scenario. See also the related work contained in Section \ref{s:related}.

\subsection{Content and contributions}\label{ss:contributions}
Our contributions can be described as follows.
\begin{enumerate}
    \item We present an efficient algorithm for pool-based batch active learning for noisy linear models (Algorithm \ref{a:batch_al_linear}). This algorithm generates pseudo-labels by computing sequences of linear classifiers that restrict their attention to exponentially small regions of the margin space, and then trains a single model based on the pseudo-labels only. The design inspiring the sampling within each stage is a G-optimal design, computed through a greedy strategy.
    We show (Theorem \ref{thm:main_linear}) that under the standard i.i.d. assumption of the (input, label) pairs, the model so trained enjoys an excess risk bound with respect to the Bayes optimal predictor which is best possible, when expressed in terms of the total number of requested labels. The number of re-training stages (that is, the number of linear classifiers computed to generate pseudo-labels) is at most logarithmic in the pool size, and automatically adapts to the noise level without knowing it in advance.
    \item Since the above algorithm does not operate on a constant batch size $B$, we show in Section \ref{ss:constant_batch_size} an easy adaptation to the constant batch size, and make the observation that $B$ therein may also scale as $T^\beta$, for some exponent $\beta < 1$ that depends on the amount of noise (see comments surrounding Corollary \ref{co:main_linear_B}), still retaining the above-mentioned optimal rates.
    \item We extend in Section \ref{s:logistic} our results to the generalized linear case (specifically, the logistic case), and point out that restricting to exponentially small regions of the margin space is also beneficial for obtaining bounds with a milder dependence on the loss curvature.
    \item Last but not least, 
    our algorithmic technique can be seen as a skeleton technique that can be applied to more general situations, provided the estimators employed at each stage and the diversity measure guiding the design have matching properties. We discuss extensions to the general nonlinear case in Section \ref{s:nonlinear}, and provide the first analysis to date that covers non-linear function spaces in batch active learning for the pool-based scenario (Theorem \ref{thm:main_nonlinear}). 
\end{enumerate}

\section{Preliminaries and Notation}\label{s:prel}
We denote by $\cX$ the input space (e.g., $\cX = \rR^d$), by $\cY$ the output space, and by $\cD$ an unknown distribution over $\cX \times \cY$. The corresponding random variables will be denoted by $\x$ and $y$. We also denote by $\cD_\cX$ the marginal distribution of $\cD$ over $\cX$. 
%
Given a function $h$ (also called a {\em hypothesis} or a {\em model}) mapping $\cX$ to $\cY$, the {\em population loss} (often referred to as {\em risk}) of $h$ is denoted by $\cL(h)$, and defined as $\cL(h) = \rE_{(\x,y) \sim \cD}[loss(h(x), y)]$, where $loss(\cdot, \cdot)\, \colon\,\cY \times \cY \to [0, 1]$ is a given {\em loss} function.
For simplicity of presentation, we restrict ourselves to a binary classification setting with 0-1 loss, so that $\cY =\{-1,+1\}$, and $loss({\hat y}, y) = \ind{{\hat y} \neq y} \in \{0,1\}$, being $\ind{\cdot}$  the indicator function of the predicate at argument. 
When clear from the surrounding context, we will omit subscripts like ``$(\x,y) \sim \cD$"
from probabilities and expectations.

We are given a class of models $\cF =\{f\,:\, \cX \rightarrow [0,1]\}$ and the Bayes optimal predictor $h^*(x) = \sign\left(f^*(x) - 1/2\right)$, where
\[
f^*(\x) = \rP(y=1|\x)
\]
is assumed to belong to class $\cF$
(the so-called {\em realizability} assumption). This assumption is reasonable whenever the model class $\cF$ we operate on is wide enough. For instance, a realizability (or quasi-realizability) assumption seems natural in overparameterized settings implemented by nowdays' Deep Neural Networks.

As a simple example, we consider a generalized linear model
\begin{equation}\label{e:logistic}
f^*(\x)=\sigma(\ang{\w^*,\x})~,
\end{equation}
where $\sigma\,:\,\rR \rightarrow [0,1]$ is a suitable sigmoidal function, e.g., $\sigma(z) = \frac{e^z}{1+e^z}$, $\w^*$ is an unknown vector in $\rR^d$, with bounded (Euclidean) norm $||\w|| \leq R$ for some $R \geq  1$,
and $\ang{\cdot,\cdot}$ denotes the usual inner product in $\rR^d$.

Throughout this paper, we adopt the commonly used low-noise condition on the marginal distribution $\cD_{\cX}$ of \citet{Tsybakov99}:
there are constant $c > 0$, $\epsilon_0\in(0, 1]$ and exponent $\alpha \geq 0$ such that for all $\epsilon \in (0,\epsilon_0]$ 
we have
\begin{equation}\label{e:tsy}
 \rP\bigl(|f^*(\x)-1/2| < \epsilon/2 \bigr) \le c\,\epsilon^\alpha~.
\end{equation}
Notice, in particular, that
$\alpha \rightarrow \infty$ gives the so-called {\em hard margin} condition
\(
 \rP\bigl(|f^*(\x)-1/2| < \epsilon \bigr) =0.
\)
%
while, at the opposite end of the spectrum, exponent $\alpha = 0$ (and $c=1$) corresponds to making {\em no assumptions whatsoever} on $\cD_{\cX}$. For simplicity, we shall assume throughout that the above low-noise condition holds for $c=1$. The noise exponent $\alpha$ and range constant $\epsilon_0$ are typically unknown, and our algorithms will not rely on the prior knowledge of them.

We are given a class of models $\cF$, and a pool $\cP$ of $T$ unlabeled instances $\x_{1}, \ldots, \x_T \in \cX$, drawn i.i.d. according to a marginal distribution $\cD_{\cX}$ obeying condition (\ref{e:tsy}) (with $c=1$). The associated labels $y_1, \ldots, y_T \in \cY$ are such that the pairs $(\x_t,y_t)$, $t = 1,\ldots, T$, are drawn i.i.d. according to $\cD$, the labels being generated according to the conditional distribution determined by some $f^* \in \cF$. The labels are not initially revealed to us, and the goal of the active learning algorithm is to come up at the end of training with a model $\hh\,:\, \cX \rightarrow \cY$ whose {\em excess risk} $\cL(\hh) - \cL(h^*)$ is as small as possible, while querying as few labels as possible in $\cP$. 

The way labels are queried follows the standard batch active learning protocol. We are given a {\em batch size} $B \geq 1$. Label acquisition and learning proceeds in a sequence of {\em stages}, $\ell = 1, 2, \ldots$. At each stage $\ell$, the algorithm is allowed to query $B$-many labels by only relying on labels acquired in the past $\ell-1$ stages.
Notice that each point $\x_t$ in pool $\cP$ can only be queried once, which is somehow equivalent to assuming that the noise in the corresponding label $y_t$ is {\em persistent}. We shall henceforth denote by $N_T(\cP)$ the total number of labels (sometimes referred to as {\em label complexity}) queried by the algorithm at hand on pool $\cP$, and by $N_{T,B}(\cP)$ the same quantity if we want to emphasize the dependence on $B$.

The analysis of our algorithms hinges upon a suitable measure of {\em diversity}, $D(\x,S)$, that quantifies how far off a data point $\x \in \cX$ is from a finite set of points $S \subseteq \cX$. Though many diversity measures may be adopted for practical purposes (e.g., \cite{wei2015submodularity,sener2018active,kvg19,a+20,killamsetty2020glister,kirsch2021simple,cm21}), the one enabling tight theoretical analyses for our algorithms is one that is somehow coupled with the estimators our active learning algorithms rely upon. Specifically, given diversity measure $D(\x,S)$, an estimator $\hf = \hf(S)$ in a fixed design scenarios is coupled with $D(\x,S)$ if we can guarantee $L_{\infty}$ approximation bounds of the form
\begin{equation}\label{e:diversity_bound}
|\hf_S(\x)-f^*(\x)| \leq D(\x,S)\qquad \forall \x~.
\end{equation}
In the case of linear function spaces over $\cX = \rR^d$, the estimators $\hf = \hf(S)$ will essentially be least-squares predictors, and a coupled diversity measure will be spectral-like:
\(
D(\x,S) = \ang{\x,\x}_{A_S^{-1}}^\frac{1}{2} = ||\x||_{A_S^{-1}} = \sqrt{\x^\top A_S^{-1}\x}~,
\)
that is, the Mahalanobis norm of $\x$ w.r.t. the positive semi-definite matrix $A_S^{-1}$, where $A_S = I + \sum_{\z \in S} \z\z^\top$~, being $I$ the $d\times d$ identity matrix.
Notice that $D(\x,S)$ is large when $\x$ is aligned with small eigenvectors of $A_S$, while it is small if $\x$ is aligned with large eigenvectors of that matrix. In particular, $D(\x,S)$ achieves its maximal value $|
|\x||^2$ when $\x$ is {\em orthogonal} to the space spanned by $S$. Hence, $\x$ is ``very different" from $S$ as measured by $D(\x,S)$ if $\x$ contributes a direction of the input space which is not already spanned by $S$. We denote by $|A_S|$ the determinant of matrix $A_S$.

In the more general nonlinear case, our diversity measure is similar in spirit to the {\em eluder dimension} \cite{rvr13}, as well as to the more recent {\em online decoupling coefficient}, as defined in \cite{d+21} -- see Section \ref{s:nonlinear} for details.

At an intuitive level, since the label requests are batched, and model updates are typically performed only at the end of each stage, a batch active learning algorithm is compelled to operate within each stage by trading off the (predicted) informativeness of the selected labels against the diversity of the data points whose labels are requested. Moreover, the larger the batch size $B$ the less adaptive the algorithm is forced to be, hence we expect $B$ to somehow play a role in the performance of the algorithm.

From a practical standpoint, there are indeed two separate notions of adaptivity to consider. One is the number of interactive rounds with the labeling oracle, the other is the number of times we {\em re-train} (or update) a model based on the labels gathered during the interactive rounds.
The two notions {\em need not} coincide. While the former essentially accounts for the cost of interacting with human labelers, the latter is more related to the cost of re-training/updating a (potentially very complex) learning system.

\section{Related work}\label{s:related}
While experimental studies on batch active learning are reported since the early 2000s (see, e.g., \cite{h+06}), it is only with the deployment at scale of Deep Neural Networks that we have seen a general resurgence of interest in active learning, and batch active learning in particular. The batch pool-based model studied here is the one that has spurred the widest attention, as it corresponds to the way in practice labels are gathered in large-scale processing pipelines. This interest has generated a flurry of recent investigations, mainly of experimental nature, yet containing a lot of interesting and diverse approaches to batch active learning. Among these are \cite{gu2012selective,gu2014batch,sener2018active,kvg19,zh19,szgw19,a+20,kpkc20,killamsetty2020glister,kirsch2021simple,gze21,cm21,kothawade2021similar}.

On the theoretical side, active learning is a well-studied sub-field of statistical learning. General references in pool-based active learning include \cite{das04,da05,ha14,no11,td17}, and specific algorithms for half-spaces under classes of input distributions are contained, e.g., in \cite{marginpaper,bl13,Zhang2021}. 
However, none of these papers tackle the practically relevant scenario of {\em batch} active learning. In fact, restricting to theoretical aspects of batch active learning makes the research landscape far less populated. Below we briefly summarize what we think are among the most relevant papers to our work, as directly related to batch active learning, and then mention recent efforts in contiguous fields, like adaptive sampling and subset selection, which may serve as a general reference and inspiration. 

Batch active learning in the pool-based scenario is one of the motivating applications in \cite{ck13}, where the main concern is to investigate general conditions under which a batch greedy policy achieves similar performance as the optimal policy that operates with the same batch size. Yet, the authors consider simple noise free scenarios, while the important observation (Theorem 2 therein) that a batch greedy algorithm is also competitive with respect to an optimal fully sequential policy (batch size one) does not apply to active learning.
\citet{chkk15,chk17} are along similar lines, with the addition of persistent noise, but do not tackle batch active learning problems.

A paper with a similar aim as ours, yet operating in the streaming setting of active learning, is \cite{acdr20}. The authors show that some classes of fully sequential active algorithms can be turned into sequential algorithms that query labels in batches and suffer only an additive (times log factors) overhead in the label complexity. This transformation is essentially obtained by freezing the state of the fully sequential algorithm, but it is unclear whether any notion of diversity over the batch is enforced by the resulting batch algorithms.

Very recent stream-based active learning papers 
that are worth mentioning are \cite{pmlr-v139-katz-samuels21a, camilleri2021selective}). These papers share similar methods and modeling assumptions as ours in leveraging optimal design, but they do not deal with batch active learning. The main concern there is essentially to improve the performance of adaptive sampling by reducing the variance of the estimators.

A learning problem similar to pool-based batch active learning is {\em training subset selection} (sometimes called dataset summarization), whose goal is to come up with a compressed version of a (big) dataset that offers to a given learning algorithm the same inference capabilities as if applied to the original dataset. The problem can be organized in rounds (as in batch active learning) and bridging one to the other can in practice be done by label hallucination/pseudo-labeling.
Representative works include \cite{wei2015submodularity,killamsetty2020glister,b+21}.

\section{The linear case}\label{s:linear}
We start off by considering a simple linear  model of the form $f^*(\x) = \frac{1+\ang{\w^*,\x}}{2}$, where both $\w^*$ and $\x$ lie in the $d$-dimensional Euclidean unit ball (so that $\ang{\w^*,\x} \in [-1,1]$ and $f^*(\x) \in [0,1]$).
Algorithm \ref{a:batch_al_linear} contains in a nutshell the main ideas behind our algorithmic solutions, which is to greedily approximate a G-optimal design in the selection of points at each stage. 
The way it is formulated, Algorithm \ref{a:batch_al_linear} does not operate with a constant batch size $B$ per stage. We will reduce to the constant batch size case in Section \ref{ss:constant_batch_size}.

\begin{algorithm2e}
\SetKwSty{textrm} 
\SetKwFor{For}{{\bf for}}{}{}
\SetKwIF{If}{ElseIf}{Else}{{\bf if}}{}{{\bf else if}}{{\bf else}}{}
\SetKwFor{While}{{\bf while}}{}{}
{\bf Input:} Confidence level $\delta \in (0,1]$, pool of instances $\cP\subseteq\rR^d$ of size $|\cP| = T$\\
{\bf Initialize:} $\cP_0 = \cP$\\
\For{$\ell=1,2,\ldots,$}{
Initialize within stage $\ell$: 
\begin{itemize}
\item $\epsilon_{\ell}=2^{-\ell}/(\sqrt{2\log\frac{2\ell(\ell+1) T}{\delta}}+1)$
\item $A_{\ell,0}=I$,\ \ $t=0$,\ \ $\cQ_\ell=\emptyset$
\end{itemize}
\While{$\cP_{\ell-1}\backslash\cQ_\ell\neq\emptyset$\ and\hspace{-0.05in}  $\max{\x\in\cP_{\ell-1}\backslash\cQ_\ell}\|\x\|_{A_{\ell,t}^{-1}} > \epsilon_\ell$}{
\begin{itemize}
\item $t = t+1$
\item Pick~$\x_{\ell, t}\in\argmax{\x\in\cP_{\ell-1}\setminus \cQ_{\ell}}\|\x\|_{A_{\ell,t-1}^{-1}}$
\item Update~~
\(
A_{\ell,t}=A_{\ell,t-1}+\x_{\ell, t}\x_{\ell, t}^\top
\)
\item $\cQ_\ell = \cQ_\ell\cup\{\x_{\ell,t}\}$
\end{itemize}
%
}
Set $T_\ell = t$, the number of queries made in stage $\ell$\\
\eIf{$\cQ_\ell\neq\emptyset$}
{
\begin{itemize}
\item Query the labels $y_{\ell, 1}, \ldots, y_{\ell, T_\ell}$ associated with
the unlabeled data in $\cQ_{\ell}$, and compute
\[
\w_{\ell}=A_{\ell, T_\ell}^{-1}\sum_{t=1}^{T_{\ell}} y_{\ell, t}\x_{\ell, t}
\]
\item Set~ $\cC_{\ell}=\{\x\in\cP_{\ell-1} \backslash \cQ_\ell: |\ang{\w_{\ell},\x}| > 2^{-\ell}\}$
\item Compute pseudo-labels on each $\x \in \cC_{\ell}$ as $\hy=\sign\ang{\w_{\ell},\x}$
\end{itemize}
}
{
\ \ \ $\w_{\ell}=\bm{0}$,~$\cC_\ell=\emptyset$
}
\ \\
Set 
$\cP_{\ell}=\cP_{\ell-1}\backslash(\cC_{\ell}\cup\cQ_\ell)$\\[2mm]
\If{$d/2^{-\ell+1}>2^{-\ell+1}|\cP_{\ell}|$}{
\begin{itemize}
\item $L=\ell$
\item Exit the for-loop ($L$ is the total number of stages)    
\end{itemize}
}
}
Predict labels in pool $\cP$:
\begin{itemize}
\item Train an SVM classifier $\hw$ on $\cup_{\ell=1}^L\cC_\ell$ via the generated pseudo-labels $\hy$
\item Predict on each $\x \in (\cup_{\ell=1}^L\cQ_\ell) \cup \cP_L$ through $\sign(\ang{\hw,\x})$
\end{itemize}
\caption{Pool-based batch active learning algorithm for linear models.}
\label{a:batch_al_linear}
\end{algorithm2e}

The algorithm takes as input a finite pool of points $\cP$ of size $T$ and proceeds across stages $\ell = 1, 2,\ldots $ by generating at each stage $\ell$ a (linear-threshold) predictor $\sign(\ang{\w_{\ell},\x})$, where $\w_\ell$ is a ridge regression estimator computed only on the labeled pairs $(\x_{\ell,1},y_{\ell,1}), \ldots, (\x_{\ell,T_\ell},y_{\ell,T_\ell})$ collected during that stage. 
These predictors are used to trim the current pool $\cP_{\ell-1}$ by eliminating both the points on which $\w_\ell$ is itself confident (set $\cC_{\ell})$ and those whose labels have just been queried (set $\cQ_{\ell}$).
At each stage $\ell$, the points $\x_{\ell,t}$ to query are selected in a greedy fashion by maximizing
\footnote{
As a matter of fact, the chosen $\x_{\ell,t}$ need not be the maximizer of $D(\x,\cQ_{\ell})$, the analysis only requires $D(\x_{\ell,t},\cQ_{\ell}) > \epsilon_\ell$.
} 
$D(\x,\cQ_{\ell}) = ||\x||_{A^{-1}_{\ell,t - 1}}$ over the current pool $\cP_{\ell-1}$ (excluding the already selected points $\cQ_{\ell}$, which are contained in $A_{\ell,t - 1}$), so as to make $\x_{\ell,t}$ maximally different from $\cQ_{\ell}$. 

When stage $\ell$ terminates, we are guaranteed that we have collected a set of points $\cQ_{\ell}$ such that all remaining points $\x$ in the pool 
satisfy $D(\x,\cQ_{\ell}) \leq \epsilon_\ell$. Threshold $\epsilon_\ell$, defined at the beginning of the stage, is exponentially decaying with $\ell$. 
It is this threshold that determines the actual length of the stage, and rules the elimination of unqueried points from the pool, along with the corresponding generation of pseudo-labels during the stage. 

Algorithm \ref{a:batch_al_linear} stops generating new stages when the size $|\cP_{\ell}|$ of pool $\cP_\ell$ triggers the condition $d/2^{-\ell+1}>2^{-\ell+1}|\cP_{\ell}|$ (which is satisfied, in particular, when $\cP_\ell$ becomes empty). In that case, the current stage $\ell$ becomes the final stage $L$. 

Finally, the algorithm uses the subset of points $\cup_{\ell=1}^L\cC_\ell$ and the associated pseudo-labels $\hy$ generated during the $L$ stages to train a linear classifier $\hw$ (e.g., an SVM) to zero empirical error on that subset. Our analysis (see Appendix \ref{sa:proofs}) shows that with high probability such a consistent linear classifier exists. Each point $\x$ that remains in the pool, that is, each $\x \in (\cup_{\ell=1}^L\cQ_\ell) \cup \cP_L$, is assigned label $\sign(\ang{\hw,\x})$. Notice, in particular, that $\hw$ is not trying to fit the queried labels of $\cup_{\ell=1}^L \cQ_\ell$, but only the pseudo-labels of $\cup_{\ell=1}^L \cC_\ell$.

It is also worth observing how Algorithm \ref{a:batch_al_linear} resolves the trade-off between informativeness and diversity we alluded to in previous sections. 
Once we reach stage $\ell$, what remains in the pool are only the points $\x$ such that $|\ang{\w_{\ell-1},\x}| \leq 2^{-\ell+1}$ (this is because we have eliminated in stage $\ell-1$ all the points in $\cC_{\ell-1}$). 
Hence, the remaining points which the approximate G-optimal design operates with in stage $\ell$ are those which the previous model $\w_{\ell-1}$ is not sufficiently confident on. The algorithm then puts all these low-confident points on the same footing (that is, they are considered equally informative if taken in isolation), and then relies on the approximate G-optimal design scheme to maximize diversity among them. The set-wise diversity measure we end up maximizing is indeed a determinant-like diversity measure. This is easily seen from the fact that
\(
\sum_{t=1}^{T_\ell} ||\x_{\ell,t}||^2_{A^{-1}_{\ell,t - 1}} \approx \log |A_{\ell,T_\ell}|~.
\)

On one hand, this careful selection of points contributes to keeping the variance of estimator $\w_{\ell}$ under control. On the other hand, the fact that we stop accumulating labels when
\(
\max{\x\in\cP_{\ell-1}\backslash\cQ_\ell}\|\x\|_{A_{\ell,T_\ell}^{-1}} \leq \epsilon_\ell
\)
essentially implies that
\(
\sign(\ang{\w_{\ell},\x}) = \sign(\ang{\w^*,\x})
\)
on all points $\x$ we generate pseudo-labels for,
which in turn ensures that these pseudo-labels are consistent with $\w^*$.

Sequential experimental design has become popular, e.g., in the (contextual) bandits literature, see Ch. 22 in \cite{ls20}, and is explicitly contained in recent works on best arm identification (e.g., \cite{NEURIPS2019_8ba6c657,cksj21}). Notice that in those works a design is a distribution over the set of actions (which would correspond to pool $\cP$ in our case), and the algorithm is afforded to sample a given action $\x_t$ {\em multiple times}, obtaining each time a fresh reward value $y_t$ such that $\rE[y_t\,|\,\x_t] = \ang{\w^*,\x_t}$. This is not conceivable in a pool-based active learning scenario where label noise is persistent, and each ``action" $\x_t$ can only be played once. This explains why the design we rely upon here is necessarily more restrained than in those papers.

\subsection{Analysis}\label{ss:analysis_linear}
The following is the main result of this section.\footnote
{
Detailed proofs are deferred to the appendices.
}
\begin{theorem}\label{thm:main_linear}
Let $T \geq d$ and assume that $\|\x\|_2 \leq 1$ for all $\x \in \cP$.
Then with probability at least $1-\delta$ over the random draw of $(\x_1,y_1),\ldots, (\x_T,y_T) \sim \cD$ the excess risk $\cL(\hw) - \cL(\w^*)$, the label complexity $N_T(\cP)$, and the number of stages $L$ generated by Algorithm \ref{a:batch_al_linear} are simultaneously upper bounded as follows:
\begin{align*}
\cL(\hw) - \cL(\w^*)
&\leq
\bar{C}
C(\delta, T, \epsilon_0)
\Biggl(
\mathrm{max}
\left\{
\left(\frac{d}{T}\right)^{\frac{\alpha+1}{\alpha+2}}
,~\frac{d}{T\epsilon_0}
\right\}
+ \frac{\log\left(\frac{\log T}{\delta}\right)}{T}
\Biggl)~,\\
N_T(\cP) 
&\leq \bar{C}
C(\delta, T, \epsilon_0)
\Biggl(
\mathrm{max}
\Biggl\{
d^{\frac{\alpha}{\alpha+2}}T^{\frac{2}{\alpha+2}},~ \frac{d}{\epsilon_0^2}
\Biggl\} + \log^2\left(\frac{\log T}{\delta}\right)
\Biggr)~,\\
L &\leq \bar{C}
\Biggl(
\mathrm{max}\left\{
\frac{\log\left(\frac{T}{d} \right)}{\alpha+2} ,~
\log\left(\frac{4}{\epsilon_0}\right)
\right\}
+ \log\left(\frac{\log T}{\delta}\right)
\Biggr)~,
\end{align*}
for an absolute constant $\bar{C}$ and 
\[
C(\delta, T, \epsilon_0) = \log^2\left(\frac{T}{\delta}\right)
\left(1 + \log^2\left(\frac{1}{\epsilon_0}\right)\right)~.
\]
\end{theorem}
\begin{proof}[Proof sketch]
We first derive a high-probability bound on the weighted empirical risk
\[
R_T(\cP) = \sum_{\x\in\cP}  \ind{\sign\ang{\hw,\x}\neq\sign\ang{\w^*,\x}}|\ang{\w^*, \x}|~,
\]
and then turn it into an excess risk bound through a uniform convergence argument. In order to bound $R_T(\cP)$, we partition the points in $\cP$ into the three subsets
\[
\cup_{\ell=1}^L \cC_\ell,\qquad \cup_{\ell=1}^L \cQ_\ell,\qquad \cP_L,
\]
and consider the contribution to $R_T(\cP)$ of each subset separately. 

When $\x \in \cup_{\ell=1}^L \cC_\ell$, we show that the pseudo-labels ${\hy}$ generated by the algorithm are with high probability consistent with those generated by $\w^*$, that is, $\sign\ang{\hw,\x} =\sign\ang{\w^*,\x}$, hence those $\x$ do not contribute to the weighted empirical risk. 

Any $\x \in \cQ_\ell$, is shown to contribute to $R_T(\cP)$ by at most $2^{-\ell}$, thus the overall contribution of $\cup_{\ell=1}^L \cQ_\ell$ can be bounded by $\sum_{\ell=1}^L T_\ell/2^{\ell}$. In turn, by the way points are picked, $T_\ell$ is roughly bounded by $d/\epsilon_\ell^2$, allowing us to conclude that the total contribution of $\cup_{\ell=1}^L \cQ_\ell$ is bounded by 
\[
\sum_{\ell=1}^L 2^{-\ell} d/\epsilon_\ell^2 \approx d/\epsilon_L~.
\]
Next, for $\x \in \cP_L$, we show that (with high probability) it must be $|\ang{\w^*,\x}| \leq 2^{-L}$ which, combined with the stopping condition defining $L$ implies an overall contribution of the same form $d/\epsilon_L$. 

Finally, since $L$ is itself a random variable, we need to devise high probability upper bounds on it. We rely on the low noise assumption (\ref{e:tsy}) to conclude that $L$ is with high probability of the form 
\[
\mathrm{max}\left\{\frac{\log(T/d)}{\alpha+2},~\log\left(\frac{4}{\epsilon_0}\right)\right\}~,
\]
which we replace back into the previous bounds yielding a guarantee of the form 
\[
R_T(\cP) \lesssim \mathrm{max}\left\{d^{\frac{\alpha+1}{\alpha+2}}\,T^{\frac{1}{\alpha+2}},~\frac{d}{\epsilon_0}\right\}~,
\]
hence an excess risk bound of the form
\[
\cL(\hw) - \cL(\w^*) \approx 
\mathrm{max}
\left\{
\left(\frac{d}{T}\right)^{\frac{\alpha+1}{\alpha+2}},\frac{d}{T\epsilon_0}
\right\}~.
\]
The analysis of the label complexity $N_T(\cP) = \sum_{\ell=1}^L T_\ell$ follows a similar pattern, but it does not require uniform convergence.
\end{proof}

\subsection{Constant batch size}\label{ss:constant_batch_size}
We now describe a simple modification to Algorithm \ref{a:batch_al_linear} that makes it work in the constant batch size case. Let us denote by $T_\ell$ the length of stage $\ell$ in Algorithm \ref{a:batch_al_linear}. The modified algorithm simply runs Algorithm \ref{a:batch_al_linear}: If $T_\ell < B$ the modified algorithm relies on model $\w_{\ell}$ generated by Algorithm \ref{a:batch_al_linear} without saturating the budget of $B$ labels at that stage. On the contrary, if $T_{\ell}\geq B$, the modified algorithm splits stage $\ell$ of Algorithm \ref{a:batch_al_linear} into $\lceil T_{\ell}/B\rceil$ stages of size $B$ (except, possibly, for the last one), and then uses the queried set $\cQ_\ell$ generated by Algorithm \ref{a:batch_al_linear} across all those stages. Hence, in this case, the modified algorithm is not exploiting the potential benefit of updating the model every $B$ queried labels. For instance, if $B = 100$ and $T_\ell = |\cQ_\ell| = 240$, the modified algorithm will split this stage into three successive stages of size 100, 100, and 40, respectively, and then rely on the 240 labels queried by Algorithm \ref{a:batch_al_linear} across the three stages. In particular, the update of the model $\w_\ell$, and the associated pseudo-label computation on sets $\cC_\ell$ is only performed at the {\em end} of the third stage.

Notice that the modified algorithm we just described is a legitimate pool-based batch active learning algorithm operating on a constant batch size $B$, and its analysis is a direct consequence of the one in Theorem \ref{thm:main_linear}, after we take care of the possible over-counting that may arise in the reduction. Specifically, observe that the final hypothesis $\hw$ produced by the modified algorithm is the same as the one computed by Algorithm \ref{a:batch_al_linear}, hence the same bound on the excess risk applies. As for label complexity, if we stipulate that a batch algorithm operating on a constant batch size $B$ will be billed $B$ labels at each stage even if it ends up querying less, then the label complexity of the modified algorithm will over-count the number of labels simply due to the rounding off in $\lceil T_{\ell}/B\rceil$. However, at each of the $L$ stages of Algorithm \ref{a:batch_al_linear}, the over-counting is bounded by $B$, so that, overall, the label complexity of the constant batch size variant exceeds the one of Algorithm \ref{a:batch_al_linear} by at most an additive $BL$ term which,
due to the bound on $L$ in Theorem \ref{thm:main_linear}, is of the form $\mathrm{max}\left\{\frac{B}{\alpha+2}\log \left(\frac{T}{d}\right),B\log \left(\frac{1}{\epsilon_0}\right)\right\}$. This is summarized in the following corollary.
\begin{corollary}\label{co:main_linear_B}
With the same assumptions and notation as in Theorem \ref{thm:main_linear}, with probability at least $1-\delta$ over the random draw of $(\x_1,y_1),\ldots, (\x_T,y_T) \sim \cD$ the label complexity $N_{T,B}(\cP)$ achieved by the modified algorithm operating on a batch of size $B$ is bounded as follows:
\begin{align*}
N_{T,B}(\cP)  
&\leq \bar{C}
C(\delta, T, \epsilon_0)
\Biggl(
\mathrm{max}
\Biggl\{
d^{\frac{\alpha}{\alpha+2}}T^{\frac{2}{\alpha+2}},~ \frac{d}{\epsilon_0^2}
\Biggl\} + \log^2\left(\frac{\log T}{\delta}\right)
\Biggr)\\
&\ \ \ + B\bar{C}\,
\Biggl(
\mathrm{max}\left\{
\frac{\log\left(\frac{T}{d} \right)}{\alpha+2} ,~
\log\left(\frac{4}{\epsilon_0}\right)
\right\}
+ \log\left(\frac{\log T}{\delta}\right)
\Biggr)~,
\end{align*}
where $\bar{C}, C(\delta, T, \epsilon_0)$ are the same as in Theorem \ref{thm:main_linear}.
\end{corollary}
A few comments are in order.
\begin{enumerate}
\item An important practical aspect of this modified algorithm (inherited from Algorithm \ref{a:batch_al_linear}) 
is the very mild number of re-trainings required to achieve the claimed performance. Despite the total number of labels can be as large as $T^{\frac{2}{\alpha+2}}$, the number $L$ of times the model is actually re-trained is not $T^{\frac{2}{\alpha+2}}/B$, but only {\em logarithmic} in $T$, irrespective of the noise level $\alpha$ (that is, even when the low-noise assumption (\ref{e:tsy}) is vacuous).
On the other hand, it is also important to observe that the bound on $L$ shrinks as $\alpha$ increases, that is, when the problem becomes easier.
%
Overall, these properties make the algorithm attractive in practical learning scenarios where the re-training time turns out to be the main bottleneck in the data acquisition process, and a learning procedure is needed that automatically adapts the re-training effort to the hardness of the problem.
\item Let us disregard lower order terms and only consider the asymptotic behavior as $T \rightarrow \infty$. Comparing the excess risk bound in Theorem \ref{thm:main_linear} to the label complexity bound in Corollary \ref{co:main_linear_B}, one can see that when $B = O(T^{\frac{2}{\alpha+2}})$ we have with high probability
\[
\cL(\hw) - \cL(\w^*) \approx \frac{1}{(N_{T,B}(\cP))^{\frac{1+\alpha}{2}}}~,
\]
which is the minimax rate one can achieve for VC classes under the low-noise condition (\ref{e:tsy}) with exponent $\alpha$ (e.g., \cite{cn08,ha09,ko10,Dekel2012SS}). 
Hence, in order to achieve high-probability minimax rates, one need not try to make the algorithm {\em more adaptive} by 
having it operate with an even smaller $B$: any $B$ as small as $T^{\frac{2}{\alpha+2}}$ will indeed suffice in our learning scenario. 
%
\item Similar minimax bounds on excess risk against label complexity have been shown in the streaming setting in \cite{Dekel2012SS,w+21}, though their results only hold in the fully sequential case (that is, $B=1$) and only hold  in expectation over the random draw of the data, not with high probability.
\end{enumerate}
The fact that a batch greedy algorithm can be competitive with a fully sequential policy has also been observed in problems which are similar in spirit to active learning, like influence maximization (see, in particular, \cite{ck13}). More recently, in the context of adaptive sequential decision making, \citet{ekm21} have proposed an efficient semi-adaptive policy that performs logarithmically-many rounds of interaction achieving similar performance as the fully sequential policy. This paper improves on the original ideas contained in \cite{gk17}. Yet,
when adapted to active learning, these results turn out to apply to very stylized scenarios that assume lack of noise in the labels, 
and/or disregard the computational aspects associated with maintaining a posterior distribution or a version space (which would be of size $O(T^d)$ in our case). 

\section{The logistic case}\label{s:logistic}
We now discuss how to extend the result of the previous section to the logistic case (the generalized linear model (\ref{e:logistic}) with $\sigma(z) = \frac{e^z}{1+e^z}$).

\begin{algorithm2e}
\SetKwSty{textrm} 
\SetKwFor{For}{{\bf for}}{}{}
\SetKwIF{If}{ElseIf}{Else}{{\bf if}}{}{{\bf else if}}{{\bf else}}{}
\SetKwFor{While}{{\bf while}}{}{}
{\bf Input:} Confidence level $\delta \in (0,1]$, pool of instances $\cP\subseteq\rR^d$ of size $|\cP| = T$, upper bound $R > 0$ on $||\w^*||$ \\
{\bf Initialize:} $\cP_0 = \cP$\\
\For{$\ell=1,2,\ldots,$}{
Initialize within stage $\ell$: 
\begin{itemize}
\item $R_\ell = R\,2^{-\ell}$
\item $\epsilon_{\ell} = R_\ell/\Bigl(16e^{8R_\ell}\sqrt{d\log\frac{2d\ell(\ell+1)}{\delta}}+ 4Re^{4R_\ell}\Bigl)$
\item $A_{\ell,0}=I$,\ $t=0$,\ $\cQ_\ell=\emptyset$
\end{itemize}
\While{$\cP_{\ell-1}\backslash\cQ_\ell\neq\emptyset$\, and  $\max{\x\in\cP_{\ell-1}\backslash\cQ_\ell}\|\x\|_{A_{\ell,t}^{-1}} > \epsilon_\ell$}{
\begin{itemize}
\item $t = t + 1$
\item Pick~$\x_{\ell, t}\in\argmax{\x\in\cP_{\ell-1}\setminus \cQ_{\ell}}\|\x\|_{A_{\ell,t-1}^{-1}}$
\item Update~~
\(
A_{\ell,t}=A_{\ell,t-1}+\x_{\ell, t}\x_{\ell, t}^\top
\)
\item $\cQ_\ell = \cQ_\ell\cup\{\x_{\ell,t}\}$
\end{itemize}
%
}
Set $T_\ell = t$,\ the number of queries made in stage $\ell$\\
\eIf{$\cQ_\ell\neq\emptyset$}
{
\begin{itemize}
\item Query the labels $y_{\ell, 1}, \ldots, y_{\ell, T_\ell}$ associated with
the unlabeled data in $\cQ_{\ell}$ 
\item Compute
$\w_\ell$ as in (\ref{e:w_ell})
%
\item Set~ $\cC_{\ell}=\{\x\in\cP_{\ell-1} \backslash \cQ_\ell: |\ang{\w_{\ell},\x}| > R_{\ell}\}$
\item Compute pseudo-labels on each $\x \in \cC_{\ell}$ as $\hy=\sign\ang{\w_{\ell},\x}$
\end{itemize}
}
{
\ \ \ $\w_{\ell}=\bm{0}$,~$\cC_\ell=\emptyset$
}
\ \\
Set 
$\cP_{\ell}=\cP_{\ell-1}\backslash(\cC_{\ell}\cup\cQ_\ell)$\\[2mm]
\If{$d/(2R_\ell) > 2R_\ell|\cP_{\ell}|$}{
\begin{itemize}
\item $L=\ell$
\item Exit the for-loop ($L$ is the total number of stages)
\end{itemize}
}
}
Predict labels in pool $\cP$:
\begin{itemize}
\item Train an SVM classifier $\hw$ on $\cup_{\ell=1}^L\cC_\ell$ via the generated pseudo-labels $\hy$
\item Predict on each $\x \in (\cup_{\ell=1}^L\cQ_\ell) \cup \cP_L$ through $\sign(\ang{\hw,\x})$
\end{itemize}
\caption{Pool-based batch active learning algorithm for logistic models.}
\label{a:batch_al_logistic}
\end{algorithm2e}

Algorithm \ref{a:batch_al_logistic} is the adaptation
to the logistic case of the algorithm of Section \ref{s:linear}, the main difference being that we now assume the comparison vector $\w^*$ to lie in a Euclidean ball of (known) radius $R$, and compute estimators $\w_\ell$ as regularized logistic regressors:
\begin{align}\label{e:w_ell}
\w_{\ell} = \argmin{\w\,:\,\max{\x \in \cQ_\ell}|\ang{\w, \x}| \leq 2R_{\ell-1}}  \Biggl[&\sum_{t=1}^{T_\ell} 
Loss(y_{\ell,t}\ang{\w, \x_{\ell,t}}) 
 + \frac{1}{8}e^{-4R_\ell}\|\w\|^2 \Biggl]~,
\end{align}
where $Loss(\cdot)$ is the logistic function
\[
Loss(a)=\log(1+e^{-a})~.
\]

One of the main concerns in the logistic case is to investigate how excess risk and label complexity bounds depend on the complexity $R$ of the comparison class. The following is the logistic counterpart to Theorem \ref{thm:main_linear}.
\begin{theorem}\label{thm:main_logistic}
Let $T \geq d$ and assume that $\|\x\|_2 \leq 1$ for all $\x \in \cP$. Then with probability at least $1-\delta$ over the random draw of $(\x_1,y_1),\ldots, (\x_T,y_T) \sim \cD$ the excess risk $\cL(\hw) - \cL(\w^*)$, the label complexity $N_T(\cP)$, and the number of stages $L$ generated by Algorithm \ref{a:batch_al_logistic} are simultaneously upper bounded as follows:
\begin{align*}
\cL(\hw) - \cL(\w^*) 
&=
\bar{C}C_{d,R}(\delta, T, \epsilon_0)
\Biggl(
    \mathrm{max}\left\{
    \left(\frac{d}{T}\right)^{\frac{\alpha+1}{\alpha+2}},~ \frac{d}{T\epsilon_0}
    \right\}+ 
    \frac{
    \log\left(\frac{\log T}{\delta}\right) + de^{8R}\lceil\log_2 R\rceil
    }{T} 
\Biggr)~, \\
N_T(\cP)
&=
    \bar{C}C_{d,R}(\delta, T, \epsilon_0)
    \Biggl(
    \mathrm{max}
    \left\{
    d^\frac{\alpha}{
    \alpha+2}T^{\frac{2}{\alpha+2}},~
    \frac{d}{\epsilon_0^2}
    \right\} +   \frac{
    \log^2\left(\frac{\log T}{\delta}\right) + de^{8R}\lceil\log_2 R\rceil
    }{T} 
    \Biggr)~,\\
L &\leq \bar{C}\Biggl(\mathrm{max}\Biggl\{
\frac{\log\left(\frac{T}{d} \right)}{\alpha+2},~\log\left(\frac{4}{\epsilon_0}\right)
\Biggl\}
+ \log\left(\frac{R\log T}{\delta}\right)
\Biggr)~,
\end{align*}
where $\bar{C}$ is an absolute constant and 
\begin{align*}
C_{d,R}(\delta, T, \epsilon_0) = \left(1+\log^2\left(\frac{1}{\epsilon_0}\right)\right)\left(
    d\log\left(\frac{T}{\delta}\right) + R^2
    \right)
    \left(
    R + \log\left(\frac{T}{\delta}\right)
    \right)~.
\end{align*}
\end{theorem}
In the above bounds, the complexity term $R$ is meant to be a constant.
Notice that the dependence on $e^R$ is common to many logistic bounds, specifically in the bandits literature. This is due to the nonlinear shape of $\sigma(\cdot)$ (see, e.g., \cite{f+10,DBLP:journals/corr/abs-1207-0166,li2017provably,DBLP:journals/corr/abs-2002-07530}, where it takes the form of an upper bound on $1/\sigma'(\cdot)$). In fact, a closer look at the multiplicative dependence on $e^R$ above reveals that this factor multiplies only {\em logarithmic} terms in $T$. This is akin to the more refined self-concordant analysis of logistic models contained in \cite{DBLP:journals/corr/abs-2002-07530}. 
Since our algorithm is focusing attention to exponentially shrinking regions of margin values $\ang{\w^*,\x}$ around the origin, we have obtained here similar guarantees without resorting to a self-concordant analysis.

A constant batch size version of Algorithm \ref{a:batch_al_logistic} can also be devised, and the associated properties spelled out. The details are very similar to those in Section \ref{ss:constant_batch_size}, and are therefore omitted.

\section{The nonlinear case}\label{s:nonlinear}
Consider now the more general non-linear scenario described in Section \ref{s:prel}.

Given an arbitrary sequence of points $S = \langle \x_1, \ldots, \x_T\rangle$ (not necessarily i.i.d. and possibly with repeated items), denote by $P(\cdot)$ a permutations over the indices $\{1,\ldots, T\}$. We define the dimension $\Dim(\cF,S)$ of the class of functions $\cF$ projected onto $S$ as
\[
\Dim(\cF,S) =\sup_{P}\sum_{t=1}^T D\Bigl(\x_{P(t)}; \langle \x_{P(1)},\ldots,\x_{P(t-1)}\rangle\Bigl),
\]
where
\[
D^2(\x, \langle \x_1,\ldots,\x_{t-1}\rangle) = \sup_{f,g\in \cF} \frac{(f(\x) - g(\x))^2}{\sum_{i=1}^{t-1} (f(\x_i) - g(\x_i))^2 + 1}~.
\]
The function $D(\x, \langle \x_1,\ldots,\x_{t-1}\rangle)$ defined above is our ($\cF$-dependent) measure of diversity of $\x$ from $\x_1, \ldots, \x_{t-1}$, and will be used to select diverse points $\x$ to query in each batch. The reader familiar with the literature in non-linear contextual bandits will recognize $\Dim(\cF,S)$ as a close relative to the so-called {\em eluder dimension} \cite{rvr13}, as well as to the more recent {\em online decoupling coefficient}, as defined in \cite{d+21}. We will give below notable examples where $\Dim(\cF,S) $ is suitably bounded. For now, what is important to keep in mind is that our active learning analyses will be non-vacuous whenever
\[
\Dim_T(\cF) = \max{S\,:\,|S|= T} \Dim(\cF,S)
\]
is {\em sub-linear} in $T$, e.g., $\Dim_T(\cF) = O(\log T)$ or $\Dim_T(\cF) = O(T^{\beta})$, for some $\beta \in [0,1)$.

The algorithm for the non-linear case is given as Algorithm \ref{a:batch_al_nonlinear}. The pseudocode is very similar to the one in Algorithm \ref{a:batch_al_linear}, where we replace the diversity measure $\|\x\|_{A_{\ell,t-1}^{-1}}$ by $D(\x,\cQ_\ell)$, the margin $\ang{\w_{\ell},\x}$ by ${\widehat f}_\ell(\x)-1/2$, and define the confidence levels through the quantity $K_T(\delta, \ell, \gamma)$ satisfying
\begin{align*}
K_T^2(\delta, \ell, \gamma) = 8 \log \frac{2\ell(\ell+1)\,N_\infty(\cF,  \cQ_\ell, \gamma)}{\delta} + 1 
+ 
2\gamma T \left(8 + \sqrt{8\log \frac{8\ell(\ell+1)T^2}{\delta}} \right)~,
\end{align*}
where $N_\infty(\cF,  \cQ_\ell, \gamma)$ is the convering number at level $\gamma > 0$ of $\cF$ w.r.t. the infinity norm over $\cQ_\ell$. Also, recall that the {\em VC-subgraph dimension} (or Pollard's {\em pseudo-dimension}) of $\cF$ is the VC-dimension of the 0/1-valued class of functions
\[
\left\{ (\x,\theta) \rightarrow \ind{f(\x) \leq \theta}\,:\, f \in \cF, \theta \in [0,1] \right\}~.
\]

\subsection{Analysis}\label{ss:analysis_nonlinear}
We have the following guarantees.
\begin{theorem}\label{thm:main_nonlinear}
Let $V$ be the VC-subgraph dimension of $\cF$, and
$T \geq \mathrm{max}\{\Dim(\cF, \cP), V\}$. 
Then with probability at least $1-\delta$ over the random draw of $(\x_1,y_1),\ldots, (\x_T,y_T) \sim \cD$ the excess risk $\cL(\widehat{f}) - \cL(f_\star)$, the label complexity $N_T(\cP)$, and the number of stages $L$ generated by Algorithm \ref{a:batch_al_nonlinear} are simultaneously upper bounded as follows:
\begin{align*}
\cL(\hf) - \cL(f_\star)
&\leq
\bar{C}
C(\delta, T, \epsilon_0)
\Biggl(
\mathrm{max}
\left\{
\left(\frac{\Dim(\cF, \cP)}{T}\right)^{\frac{\alpha+1}{\alpha+2}}
,~\frac{\Dim(\cF, \cP)}{T\epsilon_0}
\right\}\Biggl)+
\bar{C}\left(\frac{\log\left(\frac{\log T}{\delta}\right) + V \log T}{T}\right)~,\\
N_T(\cP)
&\leq \bar{C}
C(\delta, T, \epsilon_0)
\Biggl(
\mathrm{max}
\Biggl\{
\Dim(\cF, \cP)^{\frac{\alpha}{\alpha+2}}T^{\frac{2}{\alpha+2}},~ \frac{\Dim(\cF, \cP)}{\epsilon_0^2}
\Biggl\} + \log^2\left(\frac{\log T}{\delta}\right)
\Biggr)~,\\
L &\leq \bar{C}\,
\Biggl(
\mathrm{max}\left\{
\frac{\log\left(\frac{T}{\Dim(\cF, \cP)} \right)}{\alpha+2},~
\log\left(\frac{1}{\epsilon_0}\right)
\right\}  + \log\left(\frac{\log T}{\delta}\right)
\Biggr)~,
\end{align*}
for an absolute constant $\bar{C}$ and 
\[
C(\delta, T, \epsilon_0) = V\log^2\left(
\frac{T}{\delta}\right)
\left(1 + \log^2\left(\frac{1}{\epsilon_0}\right)\right)~.
\]
\end{theorem}

\begin{algorithm2e}
\SetKwSty{textrm} 
\SetKwFor{For}{{\bf for}}{}{}
\SetKwIF{If}{ElseIf}{Else}{{\bf if}}{}{{\bf else if}}{{\bf else}}{}
\SetKwFor{While}{{\bf while}}{}{}
{\bf Input:} Confidence level $\delta \in (0,1]$, pool of instances $\cP\subseteq\rR^d$ of size $|\cP| = T$\\
{\bf Initialize:} $\cP_0 = \cP$\\
\For{$\ell=1,2,\ldots,$}{
Initialize within stage $\ell$: 
\begin{itemize}
\item $\epsilon_{\ell}=2^{-\ell}/K_T(\delta, \ell, \frac{1}{T})$ 
\item $t=0$,\ \ $\cQ_\ell=\emptyset$
\end{itemize}
\While{$\cP_{\ell-1}\backslash\cQ_\ell\neq\emptyset$\ and  $\max{\x\in\cP_{\ell-1}\backslash\cQ_\ell}  D(\x, \cQ_\ell) > \epsilon_\ell$}{
\begin{itemize}
\item $t = t+1$
\item Pick~$\x_{\ell, t}\in\argmax{\x\in\cP_{\ell-1}\setminus \cQ_{\ell}}D(\x, \cQ_\ell)$
\item Update~~
$\cQ_\ell = \cQ_\ell\cup\{\x_{\ell,t}\}$
\end{itemize}
%
}
Set $T_\ell = t$, the number of queries made in stage $\ell$\\
\eIf{$\cQ_\ell\neq\emptyset$}
{
\begin{itemize}
\item Query the labels $y_{\ell, 1}, \ldots, y_{\ell, T_\ell}$ associated with
the unlabeled data in $\cQ_{\ell}$, and compute
\[
{\widehat f}_{\ell} = \argmin{f\in \cF} \sum_{t=1}^{T_\ell} \left(\frac{1+y_{\ell,t}}{2}- f(\x_{\ell,t})\right)^2
\]
\item Set~ $\cC_{\ell}=\{\x\in\cP_{\ell-1} \backslash \cQ_\ell: |{\widehat f}_\ell(\x)-1/2| > 2^{-\ell}\}$
\item Compute pseudo-labels on each $\x \in \cC_{\ell}$ as $\hy=\sign({\widehat f}_\ell(\x)-1/2)$
\end{itemize}
}
{
\ \ \ ${\widehat f}_{\ell}=1/2$ (random guess),~$\cC_\ell=\emptyset$
}
\ \\
Set 
$\cP_{\ell}=\cP_{\ell-1}\backslash(\cC_{\ell}\cup\cQ_\ell)$\\[2mm]
\If{$\Dim(\cF, \cP)/2^{-\ell+1}>2^{-\ell+1}|\cP_{\ell}|$}{
\begin{itemize}
\item $L=\ell$
\item Exit the for-loop ($L$ is the total number of stages)    
\end{itemize}
}
}
Predict labels in pool $\cP$:
\begin{itemize}
\item Train a classifier ${\widehat f}$ on $\cup_{\ell=1}^L\cC_\ell$ via the generated pseudo-labels $\hy$
\item Predict on each $\x \in (\cup_{\ell=1}^L\cQ_\ell) \cup \cP_L$ through $\sign({\widehat f}(\x)-1/2)$
\end{itemize}
\caption{Pool-based batch active learning algorithm for general non-linear models.}
\label{a:batch_al_nonlinear}
\end{algorithm2e}

\begin{remark}
For the sake of illustration, consider the case where $\Dim(\cF, \cP) \leq \Dim_T(\cF) = O(T^\beta)$, for some $\beta \in [0,1)$. In this case, one can easily see that with high probability (and disregarding constants and log factors)
\[
\cL(\hf) - \cL(f_\star) \approx \frac{1}{(N_{T}(\cP))^{\frac{(1-\beta)(1+\alpha)}{2+\beta\alpha}}}~,
\]
which are similar to the bounds obtained in \cite{w+21} in the streaming setting under assumptions that are similar in spirit.
\end{remark}
\begin{remark}
One may be wondering whether it is generally possible to bound $\Dim_T(\cF)$ in relevant cases. The only relevant cases we are currently aware of are: (i) the linear and generalized linear cases we considered in Sections \ref{s:linear} and \ref{s:logistic}, where it is easy to show that $\Dim_T(\cF) = O(d\log T)$;
(ii) The case where $\cF$ is made up of bounded functions in a RKHS. In this case,
\(
\Dim(\cF,S) = \log |I+ G(S)|~,
\)
and
\(
\Dim_T(\cF) =
\max{S\,:\,|S|= T}\log |I+ G(S)|~,
\)
where $G(S)$ is the kernel Gram matrix build out of the data in pool $S$. Hence, $\Dim_T(\cF)$ depends on the decay of the eigenvalues of the underlying kernel function. In this case, since $\cF$ may not be a VC-class, we also need to resort to known results for bounding covering numbers of function classes within a RKHS (e.g., \cite{z02}).
\end{remark}
\begin{remark}
Again, similar comments on the batch size $B$ as in the linear case can also be made for the non-linear case considered in this section.
\end{remark}
\begin{remark}
Slightly improved notions of dimension can also be considered here, which are more refined versions of $\Dim(\cF,S)$. For instance, in defining the diversity measure
$
D(\x, \langle \x_1,\ldots,\x_{t-1}\rangle)$
used by the algorithm in stage $\ell$, one can replace the sup over $f,g \in \cF$ by the sup over $f \in \cF$ only, where $g$ is replaced by $\hf_{\ell-1}$. Moreover, $f$ need not range over all $\cF$, but only over the subset of functions in $\cF$ that are empirically close to $\hf_{\ell-1}$. Consistent with this change, the associated analysis will replace $\hf_{\ell-1}$ in $\Dim(\cF,S)$ with $f_\star$. We omit all these details.
\end{remark}
%

\section{Conclusions and ongoing research}\label{s:conclusions}
We have described and analyzed novel batch active learning algorithms in the pool-based setting that achieve minimax rates of convergence of their excess risk as a function of the number of queried labels. The minimax nature of our results is retained also when the batch size $B$ is allowed to scale polynomially ($B \leq T^\beta$, for $\beta \leq 1$) with the size $T$ of the training set, the allowed exponent $\beta$ depending on the actual level of noise in the data. The algorithms have a number of re-training rounds which is at worst logarithmic, and is able to automatically adapt to the noise level. 

Our algorithms generate pseudo-labels by restricting to exponentially small regions of the margin space. In the logistic case, this has the side benefit of delivering performance bounds where the classical exponential dependence on the complexity of the comparator $\w^*$ occurs as a multiplicative factor only in logarithmic terms.

The logistic algorithm we presented in Section \ref{s:logistic} has a suboptimal dependence on the input dimension $d$ (notice the extra factor $d$ contained in $C_1$ in the excess risk bound of Theorem \ref{thm:main_logistic}), and we are currently trying to see if it is possible to achieve the same result as in the linear case. For the logistic case, a more computationally efficient algorithm actually exists that is based on the online Newton step-like analysis in \cite{DBLP:journals/corr/abs-1207-0166}. Yet, this algorithm will have a similar suboptimal dependence on $d$.

Related to the above, we are currently investigating to what extent it is possible to improve the logistic analysis so as to turn the constrained minimization problem therein into an unconstrained one. Analyses we are aware of in the contiguous field of contextual bandits in generalized linear scenarios (e.g., \cite{li2017provably}) do not seem to help, given the strong assumptions on the context distribution they formulate to achieve the optimal dependence on $d$.

We have then extended our results to the general non-linear case through a notion of dimension that takes the adaptive nature of active learning sampling into account. The resulting algorithm is a simple generalization of the one for the linear case, while the corresponding analysis is an adaptation of the linear case analysis, where we replace the greedy version of G-optimal design with a design guided by the diversity measure induced by the function space at hand.
Related to the non-linear analysis, we are currently trying to see if it is possible to replace $\Dim_T(\cF)$ by some substantially smaller notion of dimension, in both the algorithm and the analysis.

\newpage

\bibliography{paper}
\bibliographystyle{apalike}


\newpage

\appendix

\onecolumn

\section{Proofs for Section \ref{s:linear}}\label{sa:proofs}

Consider Algorithm \ref{a:batch_al_linear}, and denote by $T_\ell$ the length of stage $\ell$.

We denote for any $\epsilon>0$,
\begin{align*}
        \cT_\epsilon=\{\x\in\cP\,:\,|\ang{\w^*,\x}|\leq \epsilon\}~.
\end{align*}
Recall that in Algorithm \ref{a:batch_al_linear} variable $L$ counts the total number of stages (a random quantity), while the size of the original pool $|\cP|$ is denoted by $T$.

We first show that on the confident sets, that is, on sets $\cC_{\ell}$ where pseudo-labels are generated, the learner has with high probability no regret. Before giving our key lemma, it will be useful to define the events
\[
\ep_\ell=\left\{\max{\x\in\cP_{\ell-1}\setminus\cQ_\ell} |\ang{\w_{\ell}-\w^*, \x}|\leq 2^{-\ell}\right\}~,
\]
for $\ell=1,\ldots,L$.
\begin{lemma}\label{lma: control of gap}
For any positive $L$,  

\[
\rP\left(\bigcap_{\ell=1}^L\ep_\ell\right) > 1 - \delta~.
\]
\end{lemma}
\begin{proof}
We assume $\cP_{\ell-1}\backslash\cQ_\ell$ is not empty (it could be empty only in the final stage $L$). We follow the material contained in Chapters 20 and 21 of \citet{ls20}. Let $\xi_{\ell,t} = y_{\ell,t} - \ang{\w^*,\x_{\ell,t}}$ and notice that $\xi_{\ell,t}$ are independent 1-sub-Gaussian random variables conditioned on $\cP_{\ell-1}$. 
Also, observe that, conditioned on past stages $1, \ldots, \ell-1$, we are in a {\em fixed design} scenario, where the $\x_{\ell,t}$ are chosen without knowledge of the corresponding labels $y_{\ell,t}$. 
We can write, for any $\x\in\cP_{\ell-1}$,
\begin{align*}
\ang{\w_{\ell}-\w^*,\x}
&=
\ang{A_{\ell, T_\ell}^{-1}\left(\sum_{t=1}^{T_{\ell}} y_{\ell, t}\x_{\ell, t}\right)-\w^*,\x}\\
&=
\ang{A_{\ell, T_\ell}^{-1}\left(\sum_{t=1}^{T_{\ell}}\x_{\ell, t}\ang{\w^*,\x_{\ell, t}}+\xi_{\ell,t}\x_{\ell,t}\right)-\w^*,\x}\\
&=
\ang{A_{\ell, T_\ell}^{-1}(A_{\ell, T_\ell}-I)\w^* + A_{\ell, T_\ell}^{-1}\left(\sum_{t=1}^{T_{\ell}}\xi_{\ell,t}\x_{\ell, t}\right)-\w^*,\x}\\
&=
-\ang{\w^*, \x}_{A_{\ell, T_\ell}^{-1}} + \sum_{t=1}^{T_{\ell}}\ang{\x_{\ell, t}, \x}_{A_{\ell, T_\ell}^{-1}}\xi_{\ell,t}~.
\end{align*}
Since $\{\xi_{\ell,t}\}_{t=1}^{T_\ell}$ are 1-sub-Gaussian and independent conditioned on $\{\x_{\ell,t}\}$, the variance term $\sum_{t=1}^{T_{\ell}}\ang{\x_{\ell, t}, \x}_{A_{\ell, T_\ell}^{-1}}\xi_{\ell,t}$ is $\sqrt{\sum_{t=1}^{T_{\ell}}\ang{\x_{\ell, t}, \x}^2_{A_{\ell, T_\ell}^{-1}}}$-sub-Gaussian. We apply lemma \ref{lma: concentration of sub-gaussian r.v.}
\[
\rP\left(\Bigl|\sum_{t=1}^{T_{\ell}}\ang{\x_{\ell, t}, \x}_{A_{\ell, T_\ell}^{-1}}\xi_{\ell,t}\Bigl| \geq \sqrt{2\sum_{t=1}^{T_{\ell}}\ang{\x_{\ell, t}, \x}^2_{A_{\ell, T_\ell}^{-1}}\log\frac{2\ell(\ell+1) T}{\delta}}\right) \leq \frac{\delta}{\ell(\ell+1)T}~.
\]
Now observe that
\begin{align*}
    \sum_{t=1}^{T_{\ell}}\ang{\x_{\ell, t}, \x}^2_{A_{\ell, T_\ell}^{-1}}=\|\x\|_{A_{\ell, T_\ell}^{-1}}^2 - \|A_{\ell, T_\ell}^{-1}\x\|^2\leq \|\x\|_{A_{\ell, T_\ell}^{-1}}^2~.
\end{align*}
We plug back into the previous inequality to obtain
\[
\rP\left(\Bigl|\sum_{t=1}^{T_{\ell}}\ang{\x_{\ell, t}, \x}_{A_{\ell, T_\ell}^{-1}}\xi_{\ell,t}\Bigl| \geq \sqrt{2\|\x\|^2_{A_{\ell, T_\ell}^{-1}}\log\frac{2\ell(\ell+1)T }{\delta}}\right) \leq \frac{\delta}{\ell(\ell+1)T}~.
\]
Using a union bound, we get with probability at least $1-\frac{\delta}{\ell(\ell+1)}$,
\[
\Bigl|\sum_{t=1}^{T_{\ell}}\ang{\x_{\ell, t}, \x}_{A_{\ell, T_\ell}^{-1}}\xi_{\ell,t}\Bigr| \leq \sqrt{2\|\x\|^2_{A_{\ell, T_\ell}^{-1}}\log\frac{2\ell(\ell+1) T}{\delta}}~,
\]
holds uniformly for all $\x\in\cP_{\ell-1}$.
For the bias term $\ang{\w^*, \x}_{A_{\ell, T_\ell}^{-1}}$, notice that $A_{\ell, T_\ell} \succeq I$ implies
\begin{align*}
   |\ang{\w^*, \x}_{A_{\ell, T_\ell}^{-1}}|\leq  \|\x\|_{A_{\ell, T_\ell}^{-1}}\|\w^*\|_{A_{\ell, T_\ell}^{-1}}\leq\|\x\|_{A_{\ell, T_\ell}^{-1}}~.
\end{align*}
Hence with probability at least $1-\frac{\delta}{\ell(\ell+1)}$,
\[
|\ang{\w_{\ell}-\w^*,\x}| \leq \left(\sqrt{2\log\frac{2\ell(\ell+1)T }{\delta}}+1\right)\|\x\|_{A_{\ell, T_\ell}^{-1}}~,
\]
holds uniformly for all $\x\in\cP_{\ell-1}$. 

Notice that by the selection criterion in Algorithm \ref{a:batch_al_linear}, $\max{\x\in\cP_{\ell-1}\backslash\cQ_\ell}\|\x\|_{A_{\ell, T_\ell}^{-1}} \leq \epsilon_\ell^2$. 
As a consequence, with probability at least $1-\frac{\delta}{\ell(\ell+1)}$,
\[
\max{\x\in\cP_{\ell-1}\backslash\cQ_\ell} |\ang{\w_{\ell}-\w^*, \x}|\leq \left(\sqrt{2\log\frac{2\ell(\ell+1) T}{\delta}}+1\right)\epsilon_\ell~.
\]
Recalling the definition of $\epsilon_\ell$ in Algorithm \ref{a:batch_al_linear} and using an union bound over $\ell$, we get the desired result.
\end{proof}
As a simple consequence, we have the following lemma.
\begin{lemma}\label{lma: no regret on confidence set}
Assume $\bigcap_{\ell=1}^L\ep_\ell$ holds. Then Algorithm \ref{a:batch_al_linear} generates pseudo-labels such that, on all points $\x \in \cup_{\ell=1}^L\cC_{\ell}$, $\sign(\ang{\w_{\ell},\x}) = \sign(\ang{\w^*,\x})$.
\end{lemma}
\begin{proof}
Simply observe that if $\x \in \cup_{\ell=1}^L\cC_{\ell}$ is such that $\sign(\ang{\w_{\ell},\x}) = 1$ then $\ang{\w_{\ell},\x} > 2^{-\ell}$, which implies $\ang{\w^*,\x} > 0$ by the assumption that $\ep_\ell$ holds. Similarly, $\sign(\ang{\w_{\ell},\x}) = -1 $ implies $ \ang{\w^*,\x} <0$.
\end{proof}

\begin{lemma}\label{lma: bound of sample complexity of greedy approach}
The length $T_{\ell}$ of stage $\ell$ in Algorithm \ref{a:batch_al_linear} is (deterministically) upper bounded as
\[
T_\ell \leq \frac{8d}{\epsilon_\ell^2}\,\log\left(\frac{1}{\epsilon_\ell}\right)~.
\]
\end{lemma}
\begin{proof}
Since in stage $\ell$ the algorithm terminates at $T_\ell$, any round $t < T_{\ell} $ is such that
\[
||\x_{\ell,t+1}||^2_{A^{-1}_{\ell,t}} > \epsilon^2_\ell~.
\]
We denote $|\cdot|$ as the determinant of the matrix at argument and have the known identity
\begin{align*}
    |A_{\ell, t + 1}| = |A_{\ell, t} + \x_{\ell, t + 1}\x_{\ell, t + 1}^\top| = |A_{\ell, t}|\cdot|I + A_{\ell, t}^{-1}\x_{\ell, t + 1}\x_{\ell, t + 1}^\top| = (1 + ||\x_{\ell,t + 1}||^2_{A^{-1}_{\ell,t}})|A_{\ell, t}| \leq 2|A_{\ell, t}|~, 
\end{align*}
where the third equality holds since $I + A_{\ell, t}^{-1/2}\x_{\ell, t + 1}\x_{\ell, t + 1}^\top A_{\ell, t}^{-1/2}$ has $d - 1$ eigenvalues 1 and one eigenvalue $1 + ||\x_{\ell,t + 1}||^2_{A^{-1}_{\ell,t}}$.

Combining the above equality with the fact that $\log(1 + x) \geq \frac{x}{1 + x} \geq \frac{x}{2}$ for $0 \leq x \leq 1$, we get
\[
||\x_{\ell,t+1}||^2_{A^{-1}_{\ell,t}} = \frac{|A_{\ell, t + 1}|}{|A_{\ell, t}|} - 1
\leq 
2(\log |A_{\ell,t+1}| - \log |A_{\ell,t}|)~.
\]

Therefore,
\[
2(\log |A_{\ell,t+1}| - \log |A_{\ell,t}|) >\epsilon_\ell^2~.
\]
Summing over $t = 0,\ldots, T_\ell-1$ yields, 
\[
2\,\log \frac{|A_{\ell,T_{\ell}}|}{|A_{\ell,0}|}  \geq \epsilon_\ell^2 T_\ell~.
\]
Now, $A_{\ell,0} = I$, so that $|A_{\ell,0}|= 1$, and
\[
\log |A_{\ell,T_\ell}| \leq \log\left(\mathrm{trace}(A_{\ell,T_\ell})/d\right)^d
\leq d\log\left(1+\frac{T_\ell}{d}\right)~,
\]
yields
\[
\frac{T_\ell}{d} \leq \frac{2}{\epsilon^2_{\ell}}\,\log\left(1+\frac{T_\ell}{d}\right)~.
\]

Let $G(x)=\frac{x}{\log(1+x)}$, and notice that $G(x)$ is increasing for $x>0$. We have
\[
G\left(\frac{T_\ell}{d}\right) \leq \frac{2}{\epsilon_\ell^2} < G\left(\frac{4}{\epsilon_\ell^2}\log \frac{1}{\epsilon_\ell^2}\right)~,
\]
where the second inequality holds since $\epsilon_\ell \leq \epsilon_1 < \frac{1}{4}$. 

As a consequence,
\[
T_\ell \leq \frac{8d}{\epsilon_\ell^2}\,\log\left(\frac{1}{\epsilon_\ell}\right)~.
\]
\end{proof}
The proof then proceeds by bounding two relevant quantities associated with the behavior of Algorithm \ref{a:batch_al_linear}: the {\em label complexity}
\[
N_T(\cP) = \sum_{\ell=1}^L|\cQ_\ell|~,
\]
and the {\em weighted cumulative regret} over pool $\cP$ of size $T$, defined as
\[
R_T(\cP) = \sum_{\x\in\cP}  \ind{\sign\ang{\hw,\x}\neq\sign\ang{\w^*,\x}}|\ang{\w^*, \x}|~.
\]
We will first present intermediate bounds on $R_T(\cP)$ and $N_T(\cP)$ as a function of $L$, 
and then rely on the properties of the noise (hence the randomness on $\cP$) to complete the proofs.
To simplify the math display we denote
\[
K_T(\delta,\ell) = \sqrt{2\log\frac{2\ell(\ell+1)T}{\delta}}+1~,
\]
so that $\epsilon_\ell = \frac{1}{2^\ell K_T(\delta,\ell)}$.

Lemma \ref{lma: bound of sample complexity of greedy approach} immediately delivers the following bound on $N_T(\cP)$:
\begin{theorem}\label{thm: query bound for linear model}
For any pool realization $\cP$, the label complexity $N_T(\cP)$ of Algorithm \ref{a:batch_al_linear} operating on a pool $\cP$ of size $T$ is bounded deterministically as
\[
N_T(\cP) \leq \frac{32}{3}d\log\left(2^L K_T(\delta, L)\right)K_T^2(\delta, L)4^{L}
\]
\end{theorem}
\begin{proof}
By definition
\begin{align*}\label{e:prebound_N_T}
N_T(\cP) = \sum_{\ell=1}^L T_{\ell}
\leq& 
\sum_{\ell=1}^L \frac{8d}{\epsilon_\ell^2}\,\log\left(\frac{1}{\epsilon_\ell}\right)\\
\leq&
8d\log\left(\frac{1}{\epsilon_L}\right)K^2_T(\delta, L)\sum_{\ell=1}^L 4^\ell \\
\leq&
\frac{8}{3}d\log\left(2^L K_T(\delta, L)\right)K_T^2(\delta, L)4^{L+1}~,
\end{align*}
where the second inequality follows from the fact that both $\frac{1}{\epsilon_{\ell}}$ and $K_T(\delta, \ell)$ increase with $\ell$, and the last inequality follows from $\sum_{\ell=1}^L 4^{\ell} < \frac{4}{3} 4^L$. 
\end{proof}

As for the regret $R_T(\cP)$, we have the following high probability result.
\begin{theorem}\label{thm: regret bound for linear model}
For any pool realization $\cP$, the weighted cumulative regret $R_T(\cP)$ of Algorithm \ref{a:batch_al_linear} operating on a pool $\cP$ of size $T$ is bounded as
\[
R_T(\cP) 
\leq 
64d\log\left(2^L K_T(\delta, L)\right)K^2_T(\delta, L) 2^L + d\,2^{L-1}~,
\]
assuming $\bigcap_{\ell=1}^L\ep_\ell$ holds.
\end{theorem}
\begin{proof}
We decompose the pool $\cP$ as the union of following disjoint sets
\[
\cP= \left(\cup_{l=1}^L\cC_{\ell}\right)\cup \left(\cup_{l=1}^L\cQ_{\ell} \right) \cup\cP_L~
\]
and, correspondingly, the weighted cumulative regret $R_T(\cP)$ as the sum of the three components 
\[
R_T(\cP) = R(\cup_{l=1}^L\cC_{\ell}) + R(\cup_{l=1}^L\cQ_{\ell}) +  R(\cP_L)~.
\]
Assume $\bigcap_{\ell=1}^L\ep_\ell$ holds.  
First, notice that on $\cC_\ell$, 
\[
\sign\ang{\hw, \x}=\sign\ang{\w_\ell, \x}=\sign\ang{\w^*, \x}
\]
under the assumption that $\ep_\ell$ holds, thus points in $\cup_{\ell=1}^L\cC_{\ell}$ do not contribute weighted regret for $\hw$, i.e.,
\[
R(\cup_{l=1}^L\cC_{\ell}) = 0~.
\]
Next, on $\cP_L$, we have $|\ang{\w_L, \x}|\leq 2^{-L}$. Combining this with the assumption that $\ep_L$ holds, we get $|\ang{\w^*,\x}|\leq 2^{-L+1}$, which implies that the weighted cumulative regret on $\cP_L$ is bounded as
\[
R(\cP_L) \leq 2^{-L+1}|\cP_L| < d\,2^{L-1}~,
\]
the second inequality deriving from the stopping condition defining $L$ in Algorithm \ref{a:batch_al_linear}.

Finally, on the queried points $\cup_{l=1}^L\cQ_{\ell}$, it is unclear whether $\sign\ang{\hw,\x} = \sign\ang{\w^*,\x}$ or not, so we bound the weighted cumulative regret contribution of each data item $\x$ therein by $|\ang{\w^*,\x}|$. 
Now, by construction, $\x\in\cQ_{\ell}\subset\cP_{\ell-1}$, so that $|\ang{\w_{\ell-1}, \x}| \leq 2^{-\ell+1}$ which, combined with the assumption that $\ep_{\ell - 1}$ holds, yields  $|\ang{\w^*,\x}| \leq 2^{-\ell + 2}$. 
Since $|\cQ_{\ell}| = T_\ell$, we have 
\[
R(\cup_{\ell=1}^L\cQ_\ell) \leq 4\sum_{\ell=1}^L T_\ell\,2^{-\ell}
\]
and Lemma \ref{lma: bound of sample complexity of greedy approach} allows us to write
\[
R(\cup_{l=1}^L\cQ_{\ell}) 
\leq
32d\sum_{l=1}^L \frac{2^{-\ell}}{\epsilon_\ell^2}\,\log\left(\frac{1}{\epsilon_\ell}\right)
\leq 
64d\log\left(2^L K_T(\delta, L)\right)K^2_T(\delta, L) 2^L~,
\]
the last inequality following from a reasoning similar to the one that lead us to theorem \ref{thm: query bound for linear model}.
\end{proof}

%
Given any pool realization $\cP$, both the label complexity and weighted regret are bounded by a function of $L$. Adding the ingredient of the low noise condition (\ref{e:tsy}) helps us leverage the randomness in $\cP$ and further bound from above the number of stages $L$.

Specifically, assume the low noise condition (\ref{e:tsy}) holds for $f^*(\x) = \frac{1 + \ang{\w^*,\x}}{2}$, for some unknown exponent $\alpha \geq 0$ and unknown constant $\epsilon_0\in (0, 1]$. 
Using a multiplicative Chernoff bound, it is easy to see that for any fixed $\epsilon_*$, with probability at least $1-\delta$,
\[
|\cT_{\epsilon_*}| \leq \frac{3}{2}\left(T\epsilon_*^\alpha + \log(1/\delta)\right)~,
\]
the probability being over the random draw of the initial pool $\cP$.
Now, since $\epsilon_L$ is itself a random variable (since so is $L$), we need to resort to a covering argument. 
For any positive number $M$, consider the following set of fixed $\epsilon$ values
\begin{align*}
    \mathcal K_M = \left\{ \frac{\epsilon_0}{2^{i/\alpha}} \colon i = 0, \ldots, M \right\}~.
\end{align*}
Then with probability at least $1-\delta$, 
\[
|\cT_{\epsilon}| \leq \frac{3}{2}\left(T\epsilon^\alpha + \log\left(\frac{M}{\delta}\right)\right)~,
\]
holds simultaneously over $\epsilon\in\mathcal K_M$. Set $M = \log_2 T$ and assume $\epsilon$ is the smallest value in $\mathcal K_M$ that is bigger than or equal to $\epsilon_*$. 
If $\epsilon$ is not the smallest value in $\mathcal K_M$, then by construction we have
$\epsilon_*^\alpha \leq \epsilon^\alpha < 2\epsilon_*^\alpha$ so that, for all $\epsilon_* >  \frac{\epsilon_0}{2^{M/\alpha}}$,
\begin{equation}\label{e:uniform_epsilon}
|\cT_{\epsilon_*}| \leq |\cT_{\epsilon}| 
\leq 
\frac{3}{2}\left(T\epsilon^\alpha + \log\left(\frac{M}{\delta}\right)\right) 
< 
3\left(T\epsilon_*^\alpha + \log\left(\frac{M}{\delta}\right)\right)~. 
\end{equation}
On the other hand if $\epsilon_* \leq  \frac{\epsilon_0}{2^{M/\alpha}}$ we can write
\[
|\cT_{\epsilon_*}| \leq \Bigl|\cT_{\frac{\epsilon_0}{2^{M/\alpha}}}\Bigl| 
\leq 
\frac{3}{2}\left( \frac{T\epsilon_0^\alpha}{2^M}  + \log\left(\frac{M}{\delta}\right)\right) 
\leq
\frac{3}{2}\left(1  + \log\left(\frac{M}{\delta}\right)\right)
< 
3\log\left(\frac{M}{\delta}\right)~,
\]
making Eq. (\ref{e:uniform_epsilon}) hold in this case as well.

We define the event 
\[
\bar{\ep} = \bigcap_{\epsilon_*\in(0, \epsilon_0]}\left\{|T_{\epsilon_*}| < 3\left(T\epsilon_*^\alpha + \log\left(\frac{M}{\delta}\right)\right)\right\}~.
\]
Then
\begin{equation}\label{ineq: high prob event on x}
    \rP\left(\bar{\ep}\right) \geq 1 - \delta~, 
\end{equation}
for $M = \log_2 T$.

We set $\epsilon^*$ to be the unique solution of the equation\footnote
{
We need to further assume $T>\frac{2}{3}d$ so as to make sure the solution exists.
}
\begin{equation}\label{eqn: definition of epsilon_*}
d/\epsilon_* = 3\left(T\epsilon_*^{\alpha + 1} + \epsilon_* \log\left(\frac{M}{\delta}\right)\right)~.   
\end{equation}
Eq. (\ref{e:uniform_epsilon}) will be applied, in particular, to the margin value $2^{-L+2}$ when $2^{-L+2} \leq \epsilon_0$. 

Armed with Eqs. (\ref{e:uniform_epsilon}) and (\ref{eqn: definition of epsilon_*}) with $M = \log_2 T$, we prove a lemma that upper bounds the number of stages $L$.
\begin{lemma}\label{lma: control over epsilon_* for linear model}
Let $\epsilon_*$ be defined through (\ref{eqn: definition of epsilon_*}), with $T > \frac{2}{3}d$. Assume both $\bar{\ep}$ and $\bigcap_{\ell=1}^L \ep_\ell$ hold. Then the number of stages $L$ of Algorithm \ref{a:batch_al_linear} is upper bounded as
\begin{align*}
L 
&\leq \mathrm{max}\left( \log_2\left(\frac{1}{\epsilon_*}\right),~\log_2\left(\frac{1}{\epsilon_0}\right)\right) + 2\\
&\leq 
\mathrm{max}\left(
\log_2\left[
\left(\frac{3T}{d}\right)^{\frac{1}{\alpha+2}}
+
3\left(\frac{1}{d}\right)^{\frac{\alpha+1}{\alpha+2}}\,\left(\frac{1}{3T}\right)^{\frac{1}{\alpha+2}}\,\log(\frac{\log_2 T}{\delta})\right],~
\log_2\left(\frac{1}{\epsilon_0}\right)
\right) + 2\\ 
&= \mathrm{max}\left(
O\left(\frac{1}{\alpha+2}\log\left(\frac{T}{d} \right) + \log\left(\frac{\log T}{\delta}\right)\right),~
\log\left(\frac{4}{\epsilon_0}\right) 
\right)
~.
\end{align*}
Here the $O$-notation only omits absolute constants.
\end{lemma}
\begin{proof}
If at stage $L-1$ the algorithm has not stopped, then we must have
\begin{align*}
    d/2^{-L+2} \leq 2^{-L+2}|\cP_{L-1}|~.
\end{align*}
Notice that if $\x \in \cP_{L-1}$ then $|\ang{\w_{L-1}, \x}|\leq 2^{-L+1}$. Combining it with the assumption that $\ep_{L-1}$ holds, we have
$|\ang{\w^*,\x}|\leq 2^{-L+2}$ which implies 
$|\cP_{L-1}| \leq |\cT_{2^{-L+2}}|$.

We split the analysis into two cases. On one hand, when $2^{-L+2} >\epsilon_0$, this condition gives us directly 
\[
L \leq \log_2(\frac{1}{\epsilon_0}) + 2~.
\]
On the other hand if $2^{-L+2} \leq \epsilon_0$, then given $\bar{\ep}$ holds, $|\cT_{2^{-L+2}}| $ is upper bounded as
\[
|\cT_{2^{-L+2}}| 
\leq 
3\left(T\,2^{(-L+2)\alpha}  + \log \left(\frac{M}{\delta}\right)\right)~,
\]
with $M = \log_2 T$.
Plugging into the first display results in
\[
d/2^{-L+2} \leq 3\left(T2^{(-L+2)(\alpha + 1)} + 2^{-L+2}\log(\frac{M}{\delta})\right)~,
\]
which resembles (\ref{eqn: definition of epsilon_*}) with $2^{-L+2}$ here playing the role of $\epsilon^*$ therein. Then, from the definition of $\epsilon^*$ in (\ref{eqn: definition of epsilon_*}) we immediately obtain
%
$2^{-L+2} \geq \epsilon_*$, thus $L \leq \log_2(\frac{1}{\epsilon_*})+2$. 
Moreover,
from (\ref{eqn: definition of epsilon_*}) we see that $d/\epsilon_* \geq 3T\epsilon_*^{\alpha+1}$, which is equivalent to $\epsilon_*\leq (\frac{d}{3T})^{\frac{1}{\alpha+2}}$. 
Replacing this upper bound on $\epsilon^*$ back into the right-hand side of (\ref{eqn: definition of epsilon_*}) and dividing by $d$ yields
\[
\frac{1}{\epsilon_*} 
\leq  
\left(\frac{3T}{d}\right)^{\frac{1}{\alpha+2}}
+
3\left(\frac{1}{d}\right)^{\frac{\alpha+1}{\alpha+2}}\,\left(\frac{1}{3T}\right)^{\frac{1}{\alpha+2}}\,\log(\frac{M}{\delta})~,
\]
which gives the claimed upper bound on $L$ through $L \leq \log_2(\frac{1}{\epsilon_*})+2$.

\end{proof}

\begin{corollary}\label{thm: upper bounds, linear case}
Let $T > d$. Then with probability at least $1-2\delta$ over the random draw of $(\x_1,y_1),\ldots, (\x_T,y_T) \sim \cD$ the label complexity $N_T(\cP)$ and the weighted cumulative regret $R_T(\cP)$ of Algorithm \ref{a:batch_al_linear} simultaneously satisfy the following:
\begin{align*}
N_T(\cP) 
&= 
\log^2
\left(
\frac{T}{\delta}
\right)
\left(1 + \log^2\left(\frac{1}{\epsilon_0}\right)\right)
O\left(
\mathrm{max}
\left\{
d^{\frac{\alpha}{\alpha+2}}T^{\frac{2}{\alpha+2}},~ \frac{d}{\epsilon_0^2}
\right\}
+ \log^2\left(\frac{\log T}{\delta}\right)
\right)
\\
R_T(\cP) 
&=
\log^2
\left(
\frac{T}{\delta}
\right)
\left(1 + \log^2\left(\frac{1}{\epsilon_0}\right)\right)
O\left(
\mathrm{max}
\left\{
d^{\frac{\alpha+1}{\alpha+2}}T^{\frac{1}{\alpha+2}} ,~
\frac{d}{\epsilon_0}
\right\}
+ \log\left(\frac{\log T}{\delta}
\right)
\right)~.
\end{align*}
where the $O$-notation only omits absolute constants .
\end{corollary}
\begin{proof}
Assume both $\bar{\ep}$ and $\bigcap_{\ell=1}^L \ep_\ell$ hold. Recalling the definition of $K_T(\delta, L)$, we have
\[
K_{T}(\delta,L) = O\left(\sqrt{\log\left(\frac{T}{\delta}\right) + \log L}\right) = O\left(\sqrt{\log\left(\frac{T}{\delta}\right) + L}\right)~.
\]
Similar to lemma \ref{lma: control over epsilon_* for linear model}, we split the analysis into two cases depending on whether or not $2^{-L+2}$ is bigger than $\epsilon_0$. If $2^{-L+2} \leq \epsilon_0$, we have
\[
L \leq \log_2\left[
\left(\frac{3T}{d}\right)^{\frac{1}{\alpha+2}}
+
3\left(\frac{1}{d}\right)^{\frac{\alpha+1}{\alpha+2}}\,\left(\frac{1}{3T}\right)^{\frac{1}{\alpha+2}}\,\log(\frac{\log_2 T}{\delta})\right]~,
\]
therefore,
\[
2^L =  
O\left(\left(\frac{T}{d}\right)^{\frac{1}{\alpha+2}} + \left(\frac{1}{d}\right)^{\frac{\alpha+1}{\alpha+2}}\,\left(\frac{1}{T}\right)^{\frac{1}{\alpha+2}}\log\left(\frac{\log T}{\delta}\right)
\right)~.
\]

Plugging the above bounds into Theorem \ref{thm: query bound for linear model} gives
\begin{align*}
N_T(\cP) 
&=
O\left(d\left(L + \log K_T^2(\delta, L)\right)K_T^2(\delta, L) 4^{L}\right)\\
&=
O\left(\left(L + K_T^2(\delta, L)\right)K_T^2(\delta, L) \left(d^{\frac{\alpha}{\alpha+2}}T^{\frac{2}{\alpha+2}} + \log^2\left(\frac{\log T}{\delta}\right)\right)\right)\\
&=
O\left(\left(L + \log\left(\frac{T}{\delta}\right)\right)^2 \left(d^{\frac{\alpha}{\alpha+2}}T^{\frac{2}{\alpha+2}} + \log^2\left(\frac{\log T}{\delta}\right)\right)\right)\\
&=
\log^2
\left(
\frac{T}{\delta}\right)
O\left(
d^{\frac{\alpha}{\alpha+2}}T^{\frac{2}{\alpha+2}} + \log^2\left(\frac{\log T}{\delta}\right)
\right)~.
\end{align*}

Similarly applying them to Theorem \ref{thm: regret bound for linear model},
\begin{align*}
R_T(\cP) 
&=
O\left(d\left(L + \log K_T^2(\delta, L)\right)K_T^2(\delta, L) 2^{L}\right)\\
&=
O\left(\left(L + K_T^2(\delta, L)\right)K_T^2(\delta, L) \left(d^{\frac{\alpha+1}{\alpha+2}}T^{\frac{1}{\alpha+2}} + \log\left(\frac{\log T}{\delta}\right)\right)\right)\\
&=
O\left(\left(L + \log\left(\frac{T}{\delta}\right)\right)^2 \left(d^{\frac{\alpha+1}{\alpha+2}}T^{\frac{1}{\alpha+2}} + \log\left(\frac{\log T}{\delta}\right)\right)\right)\\
&=
\log^2\left(
\frac{T}{\delta}\right)
O\left(
d^{\frac{\alpha+1}{\alpha+2}}T^{\frac{1}{\alpha+2}} + \log\left(\frac{\log T}{\delta}\right)
\right)~,
\end{align*}
where in the second equality we used the assumption that $d<T$.

If $2^{-L+2} > \epsilon_0$, then $2^L \leq \frac{4}{\epsilon_0}$. Plugging these bounds into Theorem \ref{thm: query bound for linear model} and Theorem \ref{thm: regret bound for linear model} gives
\begin{align*}
    N_T(\cP) 
    =&
    O\left(
    \log^2
    \left(
    \frac{T}{\delta\epsilon_0}\right)\frac{d}{\epsilon_0^2}
    \right)
    =\log^2
    \left(
    \frac{T}{\delta}\right)\left(1 + \log^2\left(\frac{1}{\epsilon_0}\right)\right)
    O\left(
    \frac{d}{\epsilon_0^2}
    \right)
    \\
    R_T(\cP) 
    =&
    O\left(
    \log^2
    \left(
    \frac{T}{\delta\epsilon_0}\right)\frac{d}{\epsilon_0}\right)
    =\log^2
    \left(
    \frac{T}{\delta}\right)\left(1 + \log^2\left(\frac{1}{\epsilon_0}\right)\right)
    O\left(
    \frac{d}{\epsilon_0}
    \right)
\end{align*}

Lastly, (\ref{ineq: high prob event on x}) and lemma \ref{lma: control of gap} together yield
\[
\rP\left(\bar{\ep}\bigcap\left(\bigcap_{\ell=1}^L \ep_\ell\right)\right) \geq 1 - 2\delta~,
\]
which concludes the proof.
\end{proof}

We now turn the bound on the weighted cumulative regret $R_T(\cP)$ in the previous corollary into a bound on the excess risk. We can write
\begin{align*}
\cL(\hw) - \cL(\w^*) 
&= 
\rE_{(\x,y) \sim \cD}\Bigl[\ind{y \neq \sign(\ang{\hw,\x})} - \ind{y \neq \sign(\ang{\w^*,\x})}\Bigl]\\
&= 
\rE_{\x \sim \cD_{\cX}} \Bigl[\rE_{y \sim \cD_{\cY|\cX}}\bigl[\ind{y \neq \sign(\ang{\hw,\x})} - \ind{y \neq \sign(\ang{\w^*,\x})}\bigl]\Bigl]\\
&= 
\rE_{\x \sim \cD_{\cX}} \Bigl[\ind{\sign(\ang{\hw,\x}) \neq \sign(\ang{\w^*,\x})}\,|\ang{\w^*,\x}| \Bigl]~,
\end{align*}
where $\hw$ is the hypothesis returned by Algorithm \ref{a:batch_al_linear}. Now, simply observe that
\[
\ind{\sign(\ang{\hw,\x}) \neq \sign(\ang{\w^*,\x})}\,|\ang{\w^*,\x}| 
\]
has the same form as the function $\phi(\hw,\x)$ in Appendix \ref{as:ancillary} on which the uniform convergence result of Theorem \ref{thm:uniformconvergence} applies, with $\widehat{\epsilon}(\delta)$ therein replaced by the bound on $R_T(\cP)$ borrowed from Corollary \ref{thm: upper bounds, linear case}.
This allows us to conclude that with probability at least $1-\delta$
\[
\cL(\hw) - \cL(\w^*) 
=
\log^2\left(\frac{T}{\delta}\right)\left(1 + \log^2\left(\frac{1}{\epsilon_0}\right)\right)O\left(
\mathrm{max}
\left\{
\left(\frac{d}{T}\right)^{\frac{\alpha+1}{\alpha+2}}
,~\frac{d}{T\epsilon_0}
\right\}
+ \frac{\log\left(\frac{\log T}{\delta}\right)}{T} \right)~,
\]
as claimed in Theorem \ref{thm:main_linear} in the main body of the paper.

\section{Proofs for Section \ref{s:logistic}}\label{sa:logistic}
We adopt the same notation as in Section \ref{sa:proofs} and follow the same proof structure. 

Define the loss function
\[
Loss(a)=\log(1+e^{-a})~,
\]
and the sigmoidal function
\[
\sigma(a) = \frac{1}{1 + e^{-a}}~.
\]

The noise model in the main body of the paper can be re-formulated as follows: there exists an unknown vector $\w^*$ belonging to a Euclidean ball of radius $R \geq 1$ such that for any instance $\x$ of Euclidean norm at most 1, 
\[
\rP(y = 1\,|\,\x) 
=  
\sigma(\ang{\w^*, \x})~.
\]

Therefore we have
\[
\E{y\mid \x} = \sigma(\ang{\w^*, \x}) - \sigma(-\ang{\w^*, \x})
= 2\sigma(\ang{\w^*, \x}) - 1~,
\]
and the noise variable $\xi$ can be written as
\[
\xi:= y - \E{y \mid \x} = \frac{2y}{1+e^{y\ang{\w^*, \x}}}~. 
\]
Similar to the linear case, we denote for any $\epsilon>0$,
\begin{align*}
        \cT_\epsilon^\sigma = \{\x\in\cP\,:\,|2\sigma\left(\ang{\w^*,\x}\right) - 1|\leq \epsilon\}~.
\end{align*}

Now, recall the notation in Algorithm \ref{a:batch_al_logistic}.
Similar to $\ep_\ell$ defined in linear case, it will be useful to define the events
\[
\ep_\ell=\left\{\max{\x\in\cP_{\ell-1}\setminus\cQ_\ell} |\ang{\w_{\ell}-\w^*, \x}|\leq R_\ell\right\}~,
\]
where $R_\ell = R2^{-\ell}$ for $\ell=0,\ldots,L$.

\begin{lemma}\label{lma: control of gap, logistic model}
For any positive $L$,
\[
\rP\left(\bigcap_{\ell=1}^L\ep_\ell\right) > 1 - \delta~.
\]
\end{lemma}
\begin{proof}
We decompose the above quantity as
\begin{align*}
    \rP\left(\bigcap_{\ell=1}^L \ep_\ell\right)=\rP\left(\ep_L\,\Bigl|\,\bigcap_{\ell=1}^{L-1} \ep_\ell\right)\rP\left(\ep_{L-1}\,\Bigl|\,\bigcap_{\ell=1}^{L-2} \ep_\ell\right)\ldots\rP(\ep_2|\ep_1)\rP(\ep_1)~,
\end{align*}
and bound each factor individually. 

At the beginning of the stage $\ell$, the remaining pool is $\cP_{\ell-1}$, and
$\sup_{x \in \cP_{\ell-1}} |\ang{\w_{l-1}, \x}| \leq R_{\ell-1}$.

For $\ell \geq 2$, if $\ep_{\ell-1}$ holds then
\begin{equation}\label{e:induction}
\sup_{x \in \cP_{\ell-1}} |\ang{\w^*, \x}| \leq 2R_{\ell-1}~.
\end{equation}
Note that (\ref{e:induction}) also holds for $\ell=1$ since $\|\w^*\| \leq R$ and $\|\x\| \leq 1$.

Now, for any positive number $b$, let
\[
\Omega_\ell(b)= \{\w \in \rR^d\,:\, \max{\x \in \cQ_\ell}|\ang{\w, \x}| \leq b\},
\]
which is a convex compact set of $\w$'s.

The predictor $\w_\ell$ in Eq. (\ref{e:w_ell}) in the main body is defined as the solution of the following constraint minimization problem:
\[
\w_\ell = \argmin{\w \in \Omega_\ell(2R_{\ell-1})}  \left[\sum_{t=1}^{T_\ell} 
Loss(y_{\ell,t}\ang{\w, \x_{\ell,t}}) + \frac{1}{8}e^{-4R_\ell} \|\w\|^2 \right] ,
\]

For simplicity, from now on we omit the stage index $\ell$ from the subscripts of $\x_{\ell,t}$ and $y_{\ell,t}$ and denote $A_{\ell, T_\ell}$ as $A_\ell$. For $t = 1,\ldots, T_\ell$, denote
\begin{align*}
g_t(\w) \x_t =& \nabla_\w Loss(y_t\ang{\w, \x_t}) = -\frac{y_t}{1+ \exp(y_t\ang{\w, \x_t})}\x_t\\
h_t(\w) \x_t \x_t^\top =& \nabla_\w^2 Loss(y_t\ang{\w, \x_t}) 
= \frac{1}{2\bigl(1 + \mathrm{cosh}(y_t\ang{\w, \x_t})\bigr)}\x_t \x_t^\top .
\end{align*}
Notice that by definition 
\[
g_t(\w^*) = -\frac{1}{2}\xi_t~,
\]
where $\xi_t$ is the noise term $\xi_t = y_t - \rE[y_t\,|\,\x_t]$. Since $\mathrm{cosh}(\cdot)$ is an even function, 
\[
h_t(\w) = \frac{1}{2\bigl(1+\mathrm{cosh}(\ang{\w, \x_t})\bigr)}
\]
does not depend on $y_t$.

Since $\w^*\in\Omega_\ell(2R_{\ell-1})$ (as a consequence of (\ref{e:induction})), the assumption that $\ep_{\ell-1}$ holds and the optimality of $\w_\ell$ in $\Omega_\ell(2R_{\ell-1})$ allow us to write
\begin{align*}
  \ang{\g(\w_\ell) + \frac{1}{4}e^{-4R_\ell}\w_\ell, \w^*-\w_\ell} \geq 0~, 
\end{align*}
where
\[
\g(\w)=\sum_{t=1}^{T_\ell} g_t(\w) \x_t~.
\]
It follows that
\begin{align}\label{ineq: KKT condition}
   \ang{\g(\w^*)-\g(\w_\ell), \w^*-\w_\ell} \leq \ang{\g(\w^*), \w^*-\w_\ell} + \frac{1}{4}e^{-4R_\ell}\ang{\w_\ell, \w^* - \w_\ell}~.
\end{align}

For each $t = 1,\ldots, T_\ell$, the mean-value theorem insures the existence of a constant $\mu_\ell^t\in[0, 1]$ such that for 
\[
\w_\ell^t = (1-\mu_\ell^t)\w_\ell + \mu_\ell^t\w^*~,
\]
we have
\[
g_t(\w^*)-g_t(\w_\ell) = h_t(\w_\ell^t)  \ang{\w^*-\w_\ell, \x_t}~.
\]
Since 
\[
|\ang{\w_\ell^t, \x_t}| \leq (1-\mu_\ell^t)|\ang{\w_\ell, \x_t}| + \mu_\ell^t|\ang{\w^*, \x_t}| \leq 2R_{\ell-1} = 4R_\ell~,
\]
we have
\[
h_t(\w_\ell^t) = \frac{1}{2\bigl(1+\mathrm{cosh}(\ang{\w_\ell^t, \x_t})\bigr)} \geq \frac{1}{4}e^{-|\ang{\w_\ell^t, \x_t}|} \geq \frac{1}{4}e^{-4R_\ell}~.
\]

Introduce now the matrix
\[
H_\ell := \sum_{t=1}^{T_\ell} h_t(\w_\ell^t) \x_t \x_t^\top + \frac{1}{4}e^{-4R_\ell}I~, 
\]
where $I$ is the $d\times d$ identity matrix. We can write
\[
\g(\w^*) - \g(\w_\ell) = (H_\ell - \frac{1}{4}e^{-4R_\ell}I)(\w^* - \w_\ell)~.
\]
As a consequence, (\ref{ineq: KKT condition}) implies
\begin{align*}
\ang{H_\ell (\w^*-\w_\ell), \w^*-\w_\ell} 
&\leq \ang{\g(\w^*), \w^* - \w_\ell} +  \frac{1}{4}e^{-4R_\ell}\|\w^*-\w_\ell\|^2 + \frac{1}{4}e^{-4R_\ell}\ang{\w_\ell, \w^* - \w_\ell}\\
&=\ang{\g(\w^*) + \frac{1}{4}e^{-4R_\ell}\w^*, \w^* - \w_\ell}\\
&\leq \left(\|\g(\w^*)\|_{H^{-1}_\ell} + \frac{1}{4}e^{-4R_\ell}\|\w^*\|_{H_\ell^{-1}}\right)\|\w^*-\w_\ell\|_{H_\ell}~. 
\end{align*}

We thus obtain
\[
\|\w^*-\w_\ell\|_{H_\ell} \leq \|\g(\w^*)\|_{H_\ell^{-1}} + \frac{1}{4}e^{-4R_\ell}\|\w^*\|_{H_\ell^{-1}} \leq 4e^{4R_\ell} \|\g(\w^*)\|_{A^{-1}_\ell} + R~,
\]
where in the second inequality we used $H_\ell  \succeq \frac{1}{4}e^{-4R_\ell}A_\ell$.

To bound $\|\g(\w^*)\|_{A_\ell^{-1}}$, note that
\[
\| \g(\w^*)\|_{A_\ell^{-1}}^2 
= \| \sum_{t=1}^{T_\ell} g_t(\w^*) A^{-1/2}_\ell \x_t \|_2^2=\frac{1}{2}\| \sum_{t=1}^{T_\ell} \xi_t A^{-1/2}_\ell \x_t \|_2^2~.
\]
We plug in $A = [A_\ell^{-1/2}\x_1,\ldots,A_\ell^{-1/2}\x_{T_\ell}]$, $\xi=(\xi_1,\ldots,\xi_{T_\ell})$ into lemma \ref{lma: concentration of square of sub-gaussians} and get with probability at least $1-\frac{\delta}{\ell(\ell+1)}$, 
\[
\| \g(\w^*)\|_{A_\ell^{-1}}^2 \leq \log\frac{2d\ell(\ell+1)}{\delta}\,\tr\left(A_\ell^{-1/2}\sum_{t=1}^{T_\ell}\x_t\x_t^\top A_\ell^{-1/2}\right) = \bigl(d - \tr(A_\ell^{-1})\bigr)\log\frac{2d\ell(\ell+1)}{\delta} < d\log\frac{2d\ell(\ell+1)}{\delta}~.
\]

Thus for any $\x\in\cP_{\ell-1} \backslash \cQ_\ell$, we obtain that with probability at least $1 - \frac{\delta}{\ell(\ell+1)}$:
\begin{align*}
|\ang{\w^* -\w_\ell, \x}|
&\leq \|\x\|_{H_\ell^{-1}} \|\w^*-\w_\ell\|_{H_\ell}\\ 
&\leq 4e^{4R_\ell}\epsilon_\ell \left(4e^{4R_\ell} \|\g(\w^*)\|_{A^{-1}_\ell} + R\right) \\
&\leq \epsilon_\ell \left(16e^{8R_\ell}\sqrt{d\log\frac{2d\ell(\ell+1)}{\delta}} + 4e^{4R_\ell}R\right)~.
\end{align*}
Recalling the definition of $\epsilon_\ell$ in Algorithm \ref{a:batch_al_logistic},
we have, with probability at least $1 - \frac{\delta}{\ell(\ell+1)}$,
\[
\max{\x\in \cP_{\ell-1} \backslash \cQ_\ell}|\ang{\w_{\ell} - \w^*, \x}| \leq R_\ell~,
\]
that is, $\rP(\ep_\ell | \bigcap_{s=1}^{\ell-1}\ep_s^\sigma) \geq 1 - \frac{\delta}{\ell(\ell+1)}$ (for $\ell=1$ the above analysis gives $\rP\left(\ep_1^\sigma\right) \geq 1 - \frac{\delta}{2}$). Hence
\[
\rP\left(\bigcap_{\ell=1}^{L} \ep_\ell\right) 
\geq 
\prod_{\ell=1}^L \left(1 - \frac{\delta}{\ell(\ell+1)}\right)
\geq 
1 - \delta\sum_{\ell=1}^L\frac{1}{\ell(\ell+1)} 
> 
1 - \delta~,
\]
thereby concluding the proof.
\end{proof}

Similar to linear case, Lemma \ref{lma: no regret on confidence set} and Lemma \ref{lma: bound of sample complexity of greedy approach} also hold for logistic case.

We define the weighted cumulative regret for the logistic case as
\[
R_T(\cP) = \sum_{\x\in\cP}  \ind{\sign\ang{\hw,\x}\neq\sign\ang{\w^*,\x}}|2\sigma(\ang{\w^*, \x})-1|~,
\]
where $\hw$ is the model output by Algorithm \ref{a:batch_al_logistic}.
Notice that since $|2\sigma(x)-1| \leq |x|/2$ for all $x$, we alternatively upper bound 
\[
R_T(\cP) = \sum_{\x\in\cP}  \ind{\sign\ang{\hw,\x}\neq\sign\ang{\w^*,\x}}|\ang{\w^*, \x}|/2~,
\]
To simplify the math display we denote
\[
K_{d}(\delta,\ell) =  \sqrt{d\log\frac{2d\ell(\ell+1)}{\delta}}~,
\]
then $\epsilon_\ell = \frac{R_\ell}{16e^{8R_\ell}K_d(\delta, \ell)+ 4Re^{4R_\ell}}$. Note that here the factor $K_d(\delta, \ell)$ doesn't depend on $T$ but has a $\sqrt{d}$ dependence.

To bound the number of queries note that lemma \ref{lma: bound of sample complexity of greedy approach} still holds, we use this to prove the following result.
\begin{theorem}\label{thm: query bound for logistic model}
For any pool realization $\cP$, the label complexity $N_T(\cP)$ of Algorithm \ref{a:batch_al_logistic} operating on a pool $\cP$ of size $T$ is bounded deterministically as
\[
N_T(\cP) = d\,\max{\ell\in[L]}\log\left(\frac{1}{\epsilon_\ell}\right)O\left(K^2_d(\delta, L)e^{8R}\lceil\log_2 R\rceil + e^{4R}\lceil\log_2 R\rceil R^2 + \frac{K^2_d(\delta, L) + R^2}{R_L^2}\right)~,
\]
where the $O$-notation only omits absolute constants.
\end{theorem}

\begin{proof}
By lemma \ref{lma: bound of sample complexity of greedy approach} and the fact that $K_d(\delta, \ell)$ is an increasing function of $\ell$, we get
\begin{align*}
    T_\ell 
    &\leq 
    \frac{8d}{\epsilon_\ell^2}\,\log\left(\frac{1}{\epsilon_\ell}\right)\\ 
    &\leq 
    16d\,\frac{256e^{16R_\ell}K^2_d(\delta, L)+16R^2e^{8R_\ell}}{R_\ell^2}\,\max{\ell\in[L]}\log\left(\frac{1}{\epsilon_\ell}\right)\\
    &= 
    d\,\max{\ell\in[L]}\log\left(\frac{1}{\epsilon_\ell}\right)O\left(\frac{e^{16R_\ell}K^2_d(\delta, L)}{R_\ell^2} + \frac{R^2e^{8R_\ell}}{R_\ell^2}\right)~.
\end{align*}
where the second inequality uses $(a + b)^2 \leq 2a^2 + 2b^2$.

For the terms within the big-oh, once we sum over $\ell$ we can write
\begin{align*}
  \sum_{\ell=1}^L\frac{e^{16R_\ell}}{R_\ell^2}
  =& \sum_{R_\ell > 1}\frac{e^{16R_\ell}}{R_\ell^2} + \sum_{R_\ell \leq 1}\frac{e^{16R_\ell}}{R_\ell^2}\\
  \leq& e^{8R}\lceil\log_2 R\rceil + \frac{e^{16}}{R_L}\sum_{R_\ell \leq 1}\frac{1}{R_\ell}\\
  \leq& e^{8R}\lceil\log_2 R\rceil + \frac{2e^{16}}{R_L^2}~.
\end{align*}
And similarly
\[
\sum_{\ell=1}^L\frac{e^{8R_\ell}}{R_\ell^2} \leq e^{4R}\lceil\log_2 R\rceil + \frac{2e^8}{R_L^2} ~.
\]

Putting them together gives
\begin{align*}
    N_T(\cP) = \sum_{\ell=1}^{L} T_\ell = d\,\max{\ell\in[L]}\log\left(\frac{1}{\epsilon_\ell}\right)O\left(K^2_d(\delta, L)e^{8R}\lceil\log_2 R\rceil + e^{4R}\lceil\log_2 R\rceil R^2 + \frac{K^2_d(\delta, L) + R^2}{R_L^2}\right)~,
\end{align*}
as claimed.
\end{proof}

The following bound on the weighted cumulative regret is the logistic counterpart to Theorem \ref{thm: regret bound for linear model}.

\begin{theorem}\label{thm: regret bound for logistic model}
For any pool realization $\cP$, the weighted cumulative regret $R_T(\cP)$ of Algorithm \ref{a:batch_al_logistic} operating on a pool $\cP$ of size $T$ is bounded as
\[
R_T(\cP) = d\,\max{\ell\in[L]}\log\left(\frac{1}{\epsilon_\ell}\right) O\left(K^2_d(\delta, L)e^{8R}\lceil\log_2 R\rceil + e^{4R}\lceil\log_2 R\rceil R^2 + \frac{K^2_d(\delta, L) + R^2 }{R_L}\right)~. 
\]
assuming $\bigcap_{\ell=1}^L\ep_\ell$ holds.
\end{theorem}
\begin{proof}
We follow the same reasoning as in Theorem \ref{thm: regret bound for linear model}. We decompose the pool $\cP$ as the union of following disjoint sets
\[
\cP=\bigl(\cup_{l=1}^L\cC_{\ell}\bigr) \cup \bigl(\cup_{l=1}^L\cQ_\ell\bigr) \cup\cP_L~,
\]
and study the weighted cumulative regret components
\[
R_T(\cup_{l=1}^L\cC_{\ell})~,\quad R_T(\cup_{l=1}^L\cQ_{\ell})~,\quad R_T(\cP_L)~.
\]

Assume $\bigcap_{\ell=1}^L\ep_\ell$ holds.  
First, notice that in $\cC_\ell$, 
\[
\sign\ang{\hw, \x}=\sign\ang{\w_\ell, \x}=\sign\ang{\w^*, \x}
\]
under the assumption that $\ep_\ell$ holds, thus  $\cup_{\ell=1}^L\cC_{\ell}$ does not contribute weighted regret for $\hw$, i.e.,
\[
R_T(\cup_{l=1}^L\cC_{\ell}) = 0~.
\]
Next, on $\cP_L$, we have $|\ang{\w_L, \x}|\leq R_L$. Combining this with the assumption that $\ep_L^\sigma$ holds, we get $|\ang{\w^*,\x}|\leq 2R_L$, which implies that the weighted cumulative regret on $\cP_L$ is bounded as
\[
R_T(\cP_L) \leq R_L|\cP_L| < \frac{d}{4R_L}~,
\]
the second inequality deriving from the stopping condition defining $L$ in Algorithm \ref{a:batch_al_logistic}.

Finally, on the queried points $\cup_{l=1}^L\cQ_{\ell}$, it is unclear whether $\sign\ang{\hw,\x} = \sign\ang{\w^*,\x}$ or not, so we bound the weighted cumulative regret contribution of each data item $\x$ therein by $|\ang{\w^*,\x}|$. 
Now, by construction, $\x\in\cQ_{\ell}\subset\cP_{\ell-1}$, so that $|\ang{\w_{\ell-1}, \x}| \leq R_{\ell-1}$ which, combined with the assumption that $\ep_{\ell - 1}^\sigma$ holds, yields  $|\ang{\w^*,\x}| \leq 2R_{\ell - 1}$. 
Since $|\cQ_{\ell}| = T_\ell$, we have 
\[
R_T(\cup_{\ell=1}^L\cQ_\ell) \leq 2\sum_{\ell=1}^L T_\ell R_{\ell}
\]
and Lemma \ref{lma: bound of sample complexity of greedy approach} allows us to write
\[
R_T(\cup_{l=1}^L\cQ_{\ell}) 
\leq
16d\sum_{l=1}^L \frac{R_{\ell}}{\epsilon_\ell^2}\,\log\left(\frac{1}{\epsilon_\ell}\right)
=
d\,\max{\ell\in[L]}\log\left(\frac{1}{\epsilon_\ell}\right) O\left(\frac{e^{16R_\ell}K^2_d(\delta, L)}{R_\ell} + \frac{R^2e^{8R_\ell}}{R_\ell}\right)~.
\]
Similar to the argument in theorem \ref{thm: query bound for linear model}, we have
\begin{align*}
  \sum_{\ell=1}^L\frac{e^{16R_\ell}}{R_\ell}
  &= \sum_{R_\ell > 1}\frac{e^{16R_\ell}}{R_\ell} + \sum_{R_\ell \leq 1}\frac{e^{16R_\ell}}{R_\ell}\\
  &\leq e^{8R}\lceil\log_2 R\rceil + e^{16}\sum_{R_\ell \leq 1}\frac{1}{R_\ell}\\
  &\leq e^{8R}\lceil\log_2 R\rceil + \frac{2e^{16}}{R_L}~.
\end{align*}
and
\[
\sum_{\ell=1}^L\frac{e^{8R_\ell}}{R_\ell} \leq e^{4R}\lceil\log_2 R\rceil + \frac{2e^8}{R_L}~.
\]

Piecing together, we conclude that the total regret is bounded as
\begin{align*}
R_T(\cP) = d\,\max{\ell\in[L]}\log\left(\frac{1}{\epsilon_\ell}\right) O\left(K^2_d(\delta, L)e^{8R}\lceil\log_2 R\rceil + e^{4R}\lceil\log_2 R\rceil R^2 + \frac{K^2_d(\delta, L) + R^2}{R_L}\right)~,
\end{align*}
thereby concluding the proof.
\end{proof}

As in the linear case, adding the ingredient of the low noise condition (\ref{e:tsy}) helps us exploit the randomness in $\cP$ to further bound from above the number of stages $L$ in the logistic case.

Specifically, assume the low noise condition (\ref{e:tsy}) holds for $f^*(\x) = \sigma(\ang{\w^*,\x})$, for some unknown exponent $\alpha \geq 0$ and unknown constant $\epsilon_0\in (0, 1]$. Similar to linear case
we define the event 
\[
\bar{\ep}^\sigma = \bigcap_{\epsilon_*\in(0, \epsilon_0]}\left\{|\cT_{\epsilon_*}^\sigma| < 3\left(T\epsilon_*^\alpha + \log\left(\frac{M}{\delta}\right)\right)\right\}~.
\]
Then
\begin{equation}\label{ineq: high prob event on x, logistic case}
    \rP\left(\bar{\ep}^\sigma\right) \geq 1 - \delta~, 
\end{equation}
for $M=\log_2 T$.

%



%
\begin{lemma}\label{lma: control over epsilon_* for logistic model}
Let $\epsilon_*$ be defined through (\ref{eqn: definition of epsilon_*}), with $T > \frac{2}{3}d$. Assume both $\bar{\ep}^\sigma$ and $\bigcap_{\ell=1}^L \ep_\ell$ hold. Then the number of stages $L$ of Algorithm \ref{a:batch_al_logistic} is upper bounded as

\begin{align*}
L 
&\leq \mathrm{max}\left(\log_2\left(\frac{R}{\epsilon_*}\right),~\log_2\left(\frac{R}{\epsilon_0}\right)\right) + 2\\
&\leq 
\mathrm{max}\left(
\log_2\left[R\left(
\left(\frac{3T}{d}\right)^{\frac{1}{\alpha+2}}
+
3\left(\frac{1}{d}\right)^{\frac{\alpha+1}{\alpha+2}}\,\left(\frac{1}{3T}\right)^{\frac{1}{\alpha+2}}\,\log(\frac{\log_2 T}{\delta})\right)\right], \log_2\left(\frac{R}{\epsilon_0}\right)
\right)
+2\\ 
&= \mathrm{max}\left(
O\left(\frac{1}{\alpha+2}\log\left(\frac{T}{d} \right) + \log\left(\frac{R\log T}{\delta}\right)\right),~\log_2\left(\frac{4R}{\epsilon_0}\right)
\right)~,
\end{align*}
where the $O$-notation only hides absolute constants.
\end{lemma}
\begin{proof}
If at stage $L-1$ the algorithm has not stopped, then we must have
\begin{align*}
    d/2R_{L - 1} \leq 2R_{L - 1}|\cP_{L-1}|~.
\end{align*}
Notice that if $\x \in \cP_{L-1}$ then $|\ang{\w_{L-1}, \x}|\leq R_{L-1}$. Combining it with the assumption that $\ep_{L-1}$ holds, we have $|\ang{\w^*,\x}|\leq 2R_{L-1}$, which implies $|\cP_{L-1}| \leq |\cT_{\mathrm{tanh}(R_{L-1})}^\sigma| \leq |\cT_{2R_{L-1}}^\sigma|$. 

We split the analysis into two cases. On one hand, when $2R_{L-1} >\epsilon_0$, this condition gives us directly 
\[
L \leq \log_2(\frac{R}{\epsilon_0}) + 2~.
\]
On the other hand if $2R_{L-1} \leq \epsilon_0$, then given that $\bar{\ep}^\sigma$ holds, $|\cT_{2R_{L-1}}^\sigma|$ is upper bounded as
\[
|\cT_{2R_{L-1}}^\sigma| 
\leq 
3\left(T (2R_{L-1})^\alpha   + \log \left(\frac{M}{\delta}\right)\right)~,
\]
with $M = \log_2 T$.
Plugging into the first display results in
\[
d/2R_{L-1} \leq 3\left(T(2R_{L-1})^{\alpha + 1} + 2R_{L-1}\log(\frac{M}{\delta})\right)~,
\]
which resembles (\ref{eqn: definition of epsilon_*}) with $2R_{L-1}$ here playing the role of $\epsilon^*$ therein. Then, from the definition of $\epsilon^*$ in (\ref{eqn: definition of epsilon_*}) we immediately obtain
%
$2R_{L-1} \geq \epsilon_*$, thus $L \leq \log_2(\frac{R}{\epsilon_*})+2$. 
Moreover,
from (\ref{eqn: definition of epsilon_*}) we see that $d/\epsilon_* \geq 3T\epsilon_*^{\alpha+1}$, which is equivalent to $\epsilon_*\leq (\frac{d}{3T})^{\frac{1}{\alpha+2}}$. 
Replacing this upper bound on $\epsilon^*$ back into the right-hand side of (\ref{eqn: definition of epsilon_*}), dividing by $d$ and multiply by $R$ yields
\[
\frac{R}{\epsilon_*} 
\leq  
R\left(
\left(\frac{3T}{d}\right)^{\frac{1}{\alpha+2}}
+
3\left(\frac{1}{d}\right)^{\frac{\alpha+1}{\alpha+2}}\,\left(\frac{1}{3T}\right)^{\frac{1}{\alpha+2}}\,\log(\frac{M}{\delta})\right)~,
\]
which gives the claimed upper bound on $L$ through $L \leq \log_2(\frac{R}{\epsilon_*})+2$.
\end{proof}

\begin{corollary}\label{thm: upper bounds, logistic case}
Let $T > \frac{2}{3}d$. Then with probability at least $1-2\delta$ over the random draw of $(\x_1,y_1),\ldots, (\x_T,y_T) \sim \cD$ the label complexity $N_T(\cP)$ and the weighted cumulative regret $R_T(\cP)$ of Algorithm \ref{a:batch_al_linear} simultaneously satisfy the following:
\begin{align*}
    N_T(\cP) 
    &=
    C_{d,R}(\delta, T, \epsilon_0)
    O\left(
    \mathrm{max}
    \left\{
    d^\frac{\alpha}{
    \alpha+2}T^{\frac{2}{\alpha+2}} ,~
    \frac{d}{\epsilon_0^2}
    \right\}
    + \log^2\left(\frac{\log T}{\delta}\right) + de^{8R}\lceil\log_2 R\rceil
    \right)~,
\\
    R_T(\cP) 
    &= 
    C_{d,R}(\delta, T, \epsilon_0)
    O\left(
     \mathrm{max}
    \left\{
    d^\frac{\alpha+1}{
    \alpha+2}T^{\frac{1}{\alpha+2}},~
    \frac{d}{\epsilon_0}
    \right\}
    + \log\left(\frac{\log T}{\delta}\right) + de^{8R}\lceil\log_2 R\rceil
    \right)~,
\end{align*}
where the $O$-notation hiding absolute constants and 
\[
C_{d,R}(\delta, T, \epsilon_0) = \left(1+\log^2\left(\frac{1}{\epsilon_0}\right)\right)\left(
    d\log\left(\frac{T}{\delta}\right) + R^2
    \right)
    \left(
    R + \log\left(\frac{T}{\delta}\right)
    \right)~.
\]
\end{corollary}
\begin{proof}
Assume both $\bar{\ep}^\sigma$ and $\bigcap_{\ell=1}^L \ep_\ell$ hold. Recalling the definition of $K_d(\delta, \ell)$, we have
\begin{align*}
K_{d}(\delta,L) 
=
O\left(
\sqrt{d\log\left(\frac{d}{\delta}\right) + L}
\right)~.
\end{align*}

Similar to Lemma \ref{lma: control over epsilon_* for logistic model}, we split the analysis into two cases depending on whether or not $2R_{L-1}$ is bigger than $\epsilon_0$. If $2R_{L-1} \leq \epsilon_0$, we have
\[
L\leq\log_2\left[R\left(
\left(\frac{3T}{d}\right)^{\frac{1}{\alpha+2}}
+
3\left(\frac{1}{d}\right)^{\frac{\alpha+1}{\alpha+2}}\,\left(\frac{1}{3T}\right)^{\frac{1}{\alpha+2}}\,\log(\frac{\log_2 T}{\delta})\right)\right]~,
\]
therefore
\[
\frac{1}{R_L} =  
O\left(\left(\frac{T}{d}\right)^{\frac{1}{\alpha+2}} + \left(\frac{1}{d}\right)^{\frac{\alpha+1}{\alpha+2}}\,\left(\frac{1}{T}\right)^{\frac{1}{\alpha+2}}\log\left(\frac{\log T}{\delta}\right)\right)
~.  
\]
Moreover, we have
\[
K_{d}(\delta,L) = O\left(
\sqrt{d\log\left(\frac{d}{\delta}\right)} + \sqrt{\log \left(\frac{T}{d}\right)} + \sqrt{\log \left(\frac{R\log T}{\delta}\right)} 
\right)~.
\]
and
\begin{align*}
\max{\ell\in[L]}\log\left(\frac{1}{\epsilon_\ell}\right) 
&\leq
\log\left(
\frac{16e^{4R}K_d(\delta, L)+ 4Re^{2R}}{R_L}\right)\\
&=
O\left(R  + \log K_d(\delta, L) + \log\left(\frac{T}{d}\right)  + \log\left(\frac{R\log T}{\delta}\right)
\right)\\
&=
O\left(R + \log\left(\frac{T}{\delta}\right)
\right)~,
\end{align*}
where the last equality is because $T>\frac{2}{3}d$.

Plugging these bounds together back into factor
\[
K^2_d(\delta, L)e^{8R}\lceil\log_2 R\rceil + e^{4R}\lceil\log_2 R\rceil R^2 + \frac{K^2_d(\delta, L) + R^2 }{R_L^2}
\]
of Theorem \ref{thm: query bound for logistic model} yields
\begin{align*}
K^2_d(\delta, L)&e^{8R}\lceil\log_2 R\rceil + e^{4R}\lceil\log_2 R\rceil R^2 + \frac{K^2_d(\delta, L) + R^2}{R_L^2}\\
&=
O\left(
\left(
K^2_d(\delta, L) + R^2 
\right)
\left(
e^{8R}\lceil\log_2 R\rceil + \frac{1}{R_L^2}
\right)
\right)\\
&=
O\left(
\left(
d\log\left(\frac{T}{\delta}\right) + R^2
\right)
\left(
\left(\frac{T}{d}\right)^{\frac{2}{\alpha+2}} + \left(\frac{1}{d}\right)^{\frac{2\alpha+2}{\alpha+2}}\left(\frac{1}{T}\right)^{\frac{2}{\alpha+2}}\log^2\left(\frac{\log T}{\delta}\right) + e^{8R}\lceil\log_2 R\rceil 
\right)
\right)~,
\end{align*}
where the last equality is due to the assumption that $T > \frac{2}{3}d$. Combining the above estimates gives
\begin{align*}
    N_T(\cP) 
    =
    O\left(
    \left(
    d\log\left(\frac{T}{\delta}\right) + R^2
    \right)
    \left(
    R + \log\left(\frac{T}{\delta}\right)
    \right)
    \left(
    d^\frac{\alpha}{
    \alpha+2}T^{\frac{2}{\alpha+2}} + \log^2\left(\frac{\log T}{\delta}\right) + de^{8R}\lceil\log_2 R\rceil
    \right)
    \right)
\end{align*}

A similar argument gives
\begin{align*}
    R_T(\cP) = O\left(
   \left(
    d\log\left(\frac{T}{\delta}\right) + R^2
    \right)
    \left(
    R + \log\left(\frac{T}{\delta}\right)
    \right)
    \left(
    d^\frac{\alpha+1}{
    \alpha+2}T^{\frac{1}{\alpha+2}} + \log\left(\frac{\log T}{\delta}\right) + de^{8R}\lceil\log_2 R\rceil
    \right)
    \right)~.
\end{align*}
If $2R_{L-1} > \epsilon_0$, then $\frac{1}{R_L} \leq \frac{4}{\epsilon_0}$. Applying these bounds into Theorem \ref{thm: query bound for logistic model} we get
\begin{align*}
    N_T(\cP) 
    =&
    O\left(
    \left(
    d\log\left(\frac{T}{\delta}\right) + R^2 + \log\left(\frac{1}{\epsilon_0}\right)
    \right)
    \left(
    R + \log\left(\frac{T}{\delta\epsilon_0}\right)
    \right)
    \left(
    \frac{d}{\epsilon_0^2} + de^{8R}\lceil\log_2 R\rceil
    \right)
    \right)\\
    =&
    \left(
    d\log\left(\frac{T}{\delta}\right) + R^2 
    \right)
    \left(
    R + \log\left(\frac{T}{\delta}\right)
    \right)
    \left(1+\log^2\left(\frac{1}{\epsilon_0}\right)
    \right)
    O\left(
    \frac{d}{\epsilon_0^2} + de^{8R}\lceil\log_2 R\rceil
    \right)~.
\end{align*}
Similarly
\begin{align*}
    R_T(\cP) 
    =
    \left(
    d\log\left(\frac{T}{\delta}\right) + R^2 
    \right)
    \left(
    R + \log\left(\frac{T}{\delta}\right)
    \right)
    \left(1+\log^2\left(\frac{1}{\epsilon_0}\right)
    \right)
    O\left(
    \frac{d}{\epsilon_0} + de^{8R}\lceil\log_2 R\rceil
    \right)
    ~.
\end{align*}

Lastly, (\ref{ineq: high prob event on x, logistic case}) and Lemma \ref{lma: control of gap, logistic model} together yield
\[
\rP\left(\bar{\ep}^\sigma\bigcap\left(\bigcap_{\ell=1}^L \ep_\ell\right)\right) \geq 1 - 2\delta~,
\]
which concludes the proof.
\end{proof}

We now turn the bound on the weighted cumulative regret $R_T(\cP)$ in the previous corollary into a bound on the excess risk. As in the linear case, we have
\begin{align*}
\cL(\hw) - \cL(\w^*) 
&= 
\rE_{(\x,y) \sim \cD}\Bigl[\ind{y \neq \sign(\ang{\hw,\x})} - \ind{y \neq \sign(\ang{\w^*,\x})}\Bigl]\\
&= 
\rE_{\x \sim \cD_{\cX}} \Bigl[\rE_{y \sim \cD_{\cY|\cX}}\bigl[\ind{y \neq \sign(\ang{\hw,\x})} - \ind{y \neq \sign(\ang{\w^*,\x})}\bigl]\Bigl]\\
&= 
\rE_{\x \sim \cD_{\cX}} \Bigl[\ind{\sign(\ang{\hw,\x}) \neq \sign(\ang{\w^*,\x})}\,|2
\sigma(\ang{\w^*,\x}) - 1| \Bigl]~,
\end{align*}
where $\hw$ is the hypothesis returned by Algorithm \ref{a:batch_al_logistic}. Now, simply observe that
\[
\ind{\sign(\ang{\hw,\x}) \neq \sign(\ang{\w^*,\x})}\,|2
\sigma(\ang{\w^*,\x}) - 1| 
\]
has the same form as the function $\phi(\hw,\x)$ in Appendix \ref{as:ancillary} on which Theorem \ref{thm:uniformconvergence} applies, with $\widehat{\epsilon}(\delta)$ therein replaced by the bound on $R_T(\cP)$ deriving from Corollary \ref{thm: upper bounds, logistic case}.
This allows us to conclude that with probability at least $1-\delta$
\[
\cL(\hw) - \cL(\w^*) 
=
C_{d,R}(\delta, T, \epsilon_0)
O\left(
    \mathrm{max}\left\{
    \left(\frac{d}{T}\right)^{\frac{\alpha+1}{\alpha+2}},~ \frac{d}{T\epsilon_0}
    \right\}+ 
    \frac{
    \log\left(\frac{\log T}{\delta}\right) + de^{8R}\lceil\log_2 R\rceil
    }{T} 
\right)~,
\]
as claimed in Theorem \ref{thm:main_logistic} in the main body of the paper.

\section{Proofs for Section \ref{s:nonlinear}}\label{sa:nonlinear}

The proof has the very same structure as the one in Section \ref{sa:proofs}. Hence we only emphasize the main differences.

The starting point of the analysis is Proposition 2 in \cite{rvr13} which, with our assumptions and notation, reads as follows.
\begin{lemma}[\cite{rvr13}]\label{lma: control of quadratic difference}
Let ${\widehat f}_\ell$ be the estimator computed by Algorithm \ref{a:batch_al_nonlinear} at the end of stage $\ell$, and $\cQ_\ell = \{\x_{\ell,1},\ldots,\x_{\ell,T_\ell}\}$ be the queried points in stage $\ell$.
Then with probability at least $1-\frac{\delta}{\ell(\ell+1)}$ we have
\[
\sum_{t = 1}^{T_\ell} (f_\star(\x_{\ell,t}) - \widehat{f}_{\ell}(\x_{\ell,t}))^2 
\leq
8 \log \frac{2\ell(\ell+1)\,N_\infty(\cF,  \cQ_\ell, \gamma)}{\delta} 
+ 
2\gamma T_\ell \left(8 + \sqrt{8\log \frac{8\ell(\ell+1)T_\ell^2}{\delta}} \right)~,
\]
where $N_\infty(\cF, \cQ_\ell, \gamma)$ is the size of an $\gamma$-cover of function class $\cF$ with respect to the infinity norm.
\end{lemma}

Note that from Lemma \ref{lma: bound for covering number} we can further bound $N_\infty(\cF, \cQ_\ell, \gamma)$ by $\left(\frac{eT}{V\gamma}\right)^V$, where $d$ is the dimension of the input space and $V$ is the VC-subgraph dimension of $\cF$. To simplify the math display we denote
\[
K_T(\delta, \ell, \gamma) = \sqrt{8 \log \frac{2\ell(\ell+1)\,\left(\frac{eT}{V\gamma}\right)^V}{\delta} 
+ 
2\gamma T \left(8 + \sqrt{8\log \frac{8\ell(\ell+1)T^2}{\delta}} \right)} + 1~,
\]
so that $\epsilon_\ell = \frac{1}{2^\ell K_T(\delta,\ell,\gamma)}$.

Moreover, we define, for any $\epsilon>0$,
\begin{align*}
        \cT_\epsilon=\{\x\in\cP\,:\,|f_*(\x) - 1/2|\leq \epsilon\}~.
\end{align*}
Recall that in Algorithm \ref{a:batch_al_nonlinear} variable $L$ counts the total number of stages, while $T$ denotes the size of the original pool $\cP$.

Similar to Appendix \ref{sa:proofs}, we define the events
\[
\ep_\ell=\left\{\max{\x\in\cP_{\ell-1}\setminus\cQ_\ell} |f_*(\x)- \widehat{f}_\ell(\x)|\leq 2^{-\ell}\right\}~,
\]
for $\ell=1,\ldots,L$.
\begin{lemma}\label{lma: control of gap: nonlinear case}
For any positive $L$, we have 
\[
\rP\left(\bigcap_{\ell=1}^L\ep_\ell\right) > 1 - \delta~.
\]
\end{lemma}
\begin{proof}
We assume $\cP_{\ell-1}\backslash\cQ_\ell$ is not empty (it could be empty only in the final stage $L$). Then for any $\ell$ and $\x\in\cP_{\ell-1}\backslash\cQ_\ell$ we can write
\begin{eqnarray*}
(f_\star(\x) &-& \widehat{f}_{\ell}(\x))^2 \\
&=&
\frac{(f_\star(\x) - \widehat{f}_{\ell}(\x))^2}{\sum_{t= 1}^{T_\ell} (f_\star(\x_{\ell, t}) - \widehat{f}_{\ell}(\x_{\ell, t}))^2 + 1 } \, \left(\sum_{t= 1}^{T_\ell} (f_\star(\x_{\ell, t}) - \widehat{f}_{\ell}(\x_{\ell, t}))^2 + 1\right) \\
&\leq&
\sup_{f,g \in \cF} \frac{(f(\x) - g(\x))^2}{\sum_{t= 1}^{T_\ell} (f(\x_{\ell, t}) - g(\x_{\ell, t}))^2 + 1 }\,\left(\sum_{t= 1}^{T_\ell} (f_\star(\x_{\ell, t}) - \widehat{f}_{\ell}(\x_{\ell, t}))^2 + 1\right)\\
&=&
D^2(\x; \cQ_\ell)\,\left(\sum_{t= 1}^{T_\ell} (f_\star(\x_{\ell, t}) - \widehat{f}_{\ell}(\x_{\ell, t}))^2 + 1 \right)\\
&\leq&
\epsilon^2_\ell K^2_T\left(\delta, \ell, \gamma\right)~,
\end{eqnarray*}
holds with probability at least $1-\frac{\delta}{\ell(\ell+1)}$. This is due to lemma \ref{lma: control of quadratic difference} and the fact that $\max{\x\in\cP_{\ell-1}\backslash\cQ_\ell}  D(\x, \cQ_\ell) \leq \epsilon_\ell$.

Using the definition of $\epsilon_\ell$ and union bound we get the desired result.
\end{proof}
Similar to Appendix \ref{sa:proofs}, this yields the following consequence.
\begin{lemma}\label{lma: no regret on confidence set, nonlinear case}
Assume $\bigcap_{\ell=1}^L\ep_\ell$ holds. Then Algorithm \ref{a:batch_al_nonlinear} generates pseudo-labels such that, on all points $\x \in \cup_{\ell=1}^L\cC_{\ell}$, $\sign(\widehat{f}_\ell(\x) - 1/2) = \sign(f_*(\x) - 1/2)$.
\end{lemma}

\begin{lemma}\label{lma: bound of sample complexity of greedy approach, nonlinear case}
The length $T_{\ell}$ of stage $\ell$ in Algorithm \ref{a:batch_al_nonlinear} is (deterministically) upper bounded as
\[
T_\ell \leq \frac{\Dim(\cF,\cQ_\ell)}{\epsilon_\ell^2} \leq \frac{\Dim(\cF, \cP)}{\epsilon_\ell^2}~.
\]
\end{lemma}
\begin{proof}
Since in stage $\ell$ the algorithm terminates at $T_\ell$, any round $t < T_{\ell} $ is such that
\[
D^2\left(\x_{\ell,t+1}; \ang{\x_{\ell,1},\ldots,\x_{\ell, t}}\right) > \epsilon^2_\ell~.
\]
Summing this inequality up we get
\[
\Dim(\cF,\cQ_\ell) \geq \epsilon_\ell^2T_\ell~,
\]
as a consequence 
\[
T_\ell \leq \frac{\Dim(\cF,\cQ_\ell)}{\epsilon_\ell^2} \leq \frac{\Dim(\cF, \cP)}{\epsilon_\ell^2}~.
\]
\end{proof}
We then bound the {\em label complexity}
\[
N_T(\cP) = \sum_{\ell=1}^L|\cQ_\ell|~,
\]
and the {\em weighted cumulative regret} over pool $\cP$ of size $T$,
\[
R_T(\cP) = \sum_{\x\in\cP}  \ind{\sign(\widehat{f}_\ell(\x) - 1/2) \neq \sign(f_*(\x) - 1/2)}|f_*(\x) - 1/2|~.
\]
We will first present intermediate bounds on $R_T(\cP)$ and $N_T(\cP)$ as a function of $L$, 
and then rely on the properties of the noise (hence the randomness on $\cP$) to complete the proofs.

Lemma \ref{lma: bound of sample complexity of greedy approach, nonlinear case} immediately delivers the following bound on $N_T(\cP)$:
\begin{theorem}\label{thm: query bound for nonlinear model}
For any pool realization $\cP$, the label complexity $N_T(\cP)$ of Algorithm \ref{a:batch_al_nonlinear} operating on a pool $\cP$ of size $T$ is bounded deterministically as
\[
N_T(\cP) \leq \frac{4^{L+1}}{3}K_T^2(\delta, L, \gamma)\Dim(\cF, \cP)~.
\]
\end{theorem}

As for the regret $R_T(\cP)$, we have the following high probability result.
\begin{theorem}\label{thm: regret bound for nonlinear model}
For any pool realization $\cP$, the weighted cumulative regret $R_T(\cP)$ of Algorithm \ref{a:batch_al_linear} operating on a pool $\cP$ of size $T$ is bounded as
\[
R_T(\cP) 
\leq 
2^{L+4}K^2_T(\delta, L, \gamma)\Dim(\cF, \cP)~,
\]
assuming $\bigcap_{\ell=1}^L\ep_\ell$ holds.
\end{theorem}
\begin{proof}
We decompose the pool $\cP$ as the union of following disjoint sets
\[
\cP= \left(\cup_{l=1}^L\cC_{\ell}\right)\cup \left(\cup_{l=1}^L\cQ_{\ell} \right) \cup\cP_L~
\]
and, correspondingly, the weighted cumulative regret $R_T(\cP)$ as the sum of the three components 
\[
R_T(\cP) = R(\cup_{l=1}^L\cC_{\ell}) + R(\cup_{l=1}^L\cQ_{\ell}) +  R(\cP_L)~.
\]
Assume $\bigcap_{\ell=1}^L\ep_\ell$ holds.  
First, notice that on $\cC_\ell$, 
\[
\sign(\widehat{f}(\x) - 1/2)=\sign(\widehat{f}_\ell(\x) - 1/2)=\sign(f_*(\x) - 1/2)
\]
under the assumption that $\ep_\ell$ holds, thus points in $\cup_{\ell=1}^L\cC_{\ell}$ do not contribute weighted regret for $\widehat{f}$, i.e.,
\[
R(\cup_{l=1}^L\cC_{\ell}) = 0~.
\]
Next, on $\cP_L$, we have $ |{\widehat f}_L(\x)-1/2| \leq 2^{-L}$. Combining this with the assumption that $\ep_L$ holds, we get $|f_*(\x) - 1/2|\leq 2^{-L+1}$, which implies that the weighted cumulative regret on $\cP_L$ is bounded as
\[
R(\cP_L) \leq 2^{-L+1}|\cP_L| < \Dim(\cF, \cP)\,2^{L-1}~,
\]
the second inequality deriving from the stopping condition defining $L$ in Algorithm \ref{a:batch_al_nonlinear}.

Finally, on the queried points $\cup_{l=1}^L\cQ_{\ell}$, it is unclear whether $\sign(\widehat{f}(\x) - 1/2)=\sign(f_*(\x) - 1/2)$ or not, so we bound the weighted cumulative regret contribution of each data item $\x$ therein by $|f_*(\x) - 1/2|$. 
Now, by construction, $\x\in\cQ_{\ell}\subset\cP_{\ell-1}$, so that $|f_{\ell-1}(\x) - 1/2| \leq 2^{-\ell+1}$ which, combined with the assumption that $\ep_{\ell - 1}$ holds, yields  $|f_*(\x) - 1/2| \leq 2^{-\ell + 2}$. 
Since $|\cQ_{\ell}| = T_\ell$, we have 
\[
R(\cup_{\ell=1}^L\cQ_\ell) \leq 4\sum_{\ell=1}^L T_\ell\,2^{-\ell}
\]
and Lemma \ref{lma: bound of sample complexity of greedy approach, nonlinear case} allows us to write
\[
R(\cup_{l=1}^L\cQ_{\ell}) 
\leq
4\sum_{l=1}^L \frac{2^{-\ell}\Dim(\cF, \cP)}{\epsilon_\ell^2}
\leq 
2^{L+3}K^2_T(\delta, L, \gamma)\Dim(\cF, \cP) ~,
\]
the last inequality following from a reasoning similar to the one that lead us to theorem \ref{thm: query bound for nonlinear model}.
\end{proof}

%
Given any pool realization $\cP$, both the label complexity and weighted regret are bounded by a function of $L$. Adding the ingredient of the low noise condition (\ref{e:tsy}) helps us leverage the randomness in $\cP$ and further bound from above the number of stages $L$.

Specifically, assume the low noise condition (\ref{e:tsy}) holds for $f^*(\x)$, for some unknown exponent $\alpha \geq 0$ and unknown constant $\epsilon_0\in (0, 1]$. We set $\epsilon^*$ to be the unique solution of the equation\footnote
{
We need to further assume $T>\frac{2}{3}\Dim(\cF, \cP)$ so as to make sure the solution exists.
}
\begin{equation}\label{eqn: definition of epsilon_* for nonlinear case}
\Dim(\cF, \cP)/\epsilon_* = 3\left(T\epsilon_*^{\alpha + 1} + \epsilon_* \log\left(\frac{M}{\delta}\right)\right)~.   
\end{equation}
Eq. (\ref{e:uniform_epsilon}) holds in this case by low noise condition and will be applied, in particular, to the margin value $2^{-L+2}$ when $2^{-L+2} \leq \epsilon_0$. 

Armed with Eqs. (\ref{e:uniform_epsilon}) and (\ref{eqn: definition of epsilon_* for nonlinear case}) with $M = \log_2 T$, we prove a lemma that upper bounds the number of stages $L$. Recall the definition of event $\bar{\ep}$:
\[
\bar{\ep} = \bigcap_{\epsilon_*\in(0, \epsilon_0]}\left\{|T_{\epsilon_*}| < 3\left(T\epsilon_*^\alpha + \log\left(\frac{M}{\delta}\right)\right)\right\}~.
\] 
\begin{lemma}\label{lma: control over epsilon_* for nonlinear model}
Let $\epsilon_*$ be defined through (\ref{eqn: definition of epsilon_*}), with $T > \frac{2}{3}\Dim(\cF, \cP)$. 
Assume both $\bar{\ep}$ and $\bigcap_{\ell=1}^L \ep_\ell$ hold. Then the number of stages $L$ of Algorithm \ref{a:batch_al_linear} is upper bounded as
\begin{align*}
L 
&\leq \mathrm{max}\left( \log_2\left(\frac{1}{\epsilon_*}\right),~\log_2\left(\frac{1}{\epsilon_0}\right)\right) + 2\\
&\leq 
\mathrm{max}\left(
\log_2\left[
\left(\frac{3T}{\Dim(\cF, \cP)}\right)^{\frac{1}{\alpha+2}}
+
3\left(\frac{1}{\Dim(\cF, \cP)}\right)^{\frac{\alpha+1}{\alpha+2}}\,\left(\frac{1}{3T}\right)^{\frac{1}{\alpha+2}}\,\log(\frac{\log_2 T}{\delta})\right],~
\log_2\left(\frac{1}{\epsilon_0}\right)
\right) + 2\\ 
&= \mathrm{max}\left(
O\left(\frac{1}{\alpha+2}\log\left(\frac{T}{\Dim(\cF, \cP)} \right) + \log\left(\frac{\log T}{\delta}\right)\right),~
\log\left(\frac{4}{\epsilon_0}\right) 
\right)
~.
\end{align*}
Here the $O$-notation only omits absolute constants.
\end{lemma}
\begin{proof}
If at stage $L-1$ the algorithm has not stopped, then we must have
\begin{align*}
    \Dim(\cF, \cP)/2^{-L+2} \leq 2^{-L+2}|\cP_{L-1}|~.
\end{align*}
Notice that if $\x \in \cP_{L-1}$ then $|\widehat{f}_{L-1}(\x) - 1/2|\leq 2^{-L+1}$. Combining it with the assumption that $\ep_{L-1}$ holds, we have
$|f_*(\x) - 1/2|\leq 2^{-L+2}$ which implies 
$|\cP_{L-1}| \leq |\cT_{2^{-L+2}}|$.

The rest of the proof remains unchanged compared to Lemma \ref{lma: control over epsilon_* for linear model}, except that we replace $d$ by $\Dim(\cF, \cP)$.
\end{proof}

\begin{corollary}\label{thm: upper bounds, nonlinear case}
Let $T > \Dim(\cF, \cP)$. Then with probability at least $1-2\delta$ over the random draw of $(\x_1,y_1),\ldots, (\x_T,y_T) \sim \cD$ the label complexity $N_T(\cP)$ and the weighted cumulative regret $R_T(\cP)$ of Algorithm \ref{a:batch_al_nonlinear} simultaneously satisfy the following:
\begin{align*}
N_T(\cP) 
&= 
V\log^2\left(\frac{T}{\delta}\right)
\left(1 + \log^2\left(\frac{1}{\epsilon_0}\right)\right)
O\left(
\mathrm{max}
\left\{
\Dim(\cF, \cP)^{\frac{\alpha}{\alpha+2}}T^{\frac{2}{\alpha+2}},~ \frac{\Dim(\cF, \cP)}{\epsilon_0^2}
\right\}
+ \log^2\left(\frac{\log T}{\delta}\right)
\right)
\\
R_T(\cP) 
&=
V\log^2\left(\frac{T}{\delta}\right)
\left(1 + \log^2\left(\frac{1}{\epsilon_0}\right)\right)
O\left(
\mathrm{max}
\left\{
\Dim(\cF, \cP)^{\frac{\alpha+1}{\alpha+2}}T^{\frac{1}{\alpha+2}} ,~
\frac{\Dim(\cF, \cP)}{\epsilon_0}
\right\}
+ \log\left(\frac{\log T}{\delta}
\right)
\right)~.
\end{align*}

where the $O$-notation only omits absolute constants .
\end{corollary}
\begin{proof}
Assume both $\bar{\ep}$ and $\bigcap_{\ell=1}^L \ep_\ell$ hold. Recalling the definition of $K_T(\delta, L, \gamma)$, we have
\begin{align*}
K_{T}(\delta, L, \gamma) &= \sqrt{8 \log \frac{2\ell(\ell+1)\,\left(\frac{eT}{V\gamma}\right)^V}{\delta} 
+ 
2\gamma T \left(8 + \sqrt{8\log \frac{8\ell(\ell+1)T^2}{\delta}} \right)} + 1\\
&= O\left(\sqrt{\log\left(\frac{\left(\frac{eT}{V\gamma}\right)^V}{\delta}\right) + \gamma T \sqrt{\log L + \log\left(\frac{T}{\delta}\right)} + \log L}\right)~.
\end{align*}
We choose $\gamma=\frac{1}{T}$ and simplify the above display as
\begin{align*}
K_{T}(\delta, L, \frac{1}{T})=  O\left(\sqrt{V\log\left(\frac{T}{\delta}\right) + \log L}\right) = O\left(\sqrt{V\log\left(\frac{T}{\delta}\right) + L}\right)~.
\end{align*}

Similar to lemma \ref{lma: control over epsilon_* for nonlinear model}, we split the analysis into two cases depending on whether or not $2^{-L+2}$ is bigger than $\epsilon_0$. If $2^{-L+2} \leq \epsilon_0$, we have
\[
L \leq \log_2\left[
\left(\frac{3T}{\Dim(\cF, \cP)}\right)^{\frac{1}{\alpha+2}}
+
3\left(\frac{1}{\Dim(\cF, \cP)}\right)^{\frac{\alpha+1}{\alpha+2}}\,\left(\frac{1}{3T}\right)^{\frac{1}{\alpha+2}}\,\log(\frac{\log_2 T}{\delta})\right]~,
\]
therefore,
\[
2^L =  
O\left(\left(\frac{T}{\Dim(\cF, \cP)}\right)^{\frac{1}{\alpha+2}} + \left(\frac{1}{\Dim(\cF, \cP)}\right)^{\frac{\alpha+1}{\alpha+2}}\,\left(\frac{1}{T}\right)^{\frac{1}{\alpha+2}}\log\left(\frac{\log T}{\delta}\right)
\right)~.
\]

Plugging the above bounds into Theorem \ref{thm: query bound for nonlinear model} gives
\begin{align*}
N_T(\cP) 
&=
O\left(\Dim(\cF, \cP)\left(L + \log K_T\left(\delta, L, \frac{1}{T}\right)\right)K_T^2(\delta, L, \frac{1}{T}) 4^{L}\right)\\
&=
O\left(\left(L + \log K_T\left(\delta, L, \frac{1}{T}\right)\right)K_T^2(\delta, L, \frac{1}{T}) \left(\Dim(\cF, \cP)^{\frac{\alpha}{\alpha+2}}T^{\frac{2}{\alpha+2}} + \log^2\left(\frac{\log T}{\delta}\right)\right)\right)\\
&=
O\left(\left(L + V\log\left(\frac{T}{\delta}\right)\right)\left(L + \log\left(\frac{T}{\delta}\right)\right) \left(\Dim(\cF, \cP)^{\frac{\alpha}{\alpha+2}}T^{\frac{2}{\alpha+2}} + \log^2\left(\frac{\log T}{\delta}\right)\right)\right)\\
&=
\left(L + V\log\left(\frac{T}{\delta}\right)\right)\left(L + \log\left(\frac{T}{\delta}\right)\right)
O\left(
\Dim(\cF, \cP)^{\frac{\alpha}{\alpha+2}}T^{\frac{2}{\alpha+2}} + \log^2\left(\frac{\log T}{\delta}\right)
\right)~.
\end{align*}

Similarly applying them to Theorem \ref{thm: regret bound for nonlinear model},
\begin{align*}
R_T(\cP) 
&=
O\left(\Dim(\cF, \cP)\left(L + \log K_T\left(\delta, L, \frac{1}{T}\right)\right)K_T^2(\delta, L, \frac{1}{T}) 2^{L}\right)\\
&=
O\left(\left(L + K_T\left(\delta, L, \frac{1}{T}\right)\right)K_T^2(\delta, L, \frac{1}{T}) \left(\Dim(\cF, \cP)^{\frac{\alpha+1}{\alpha+2}}T^{\frac{1}{\alpha+2}} + \log\left(\frac{\log T}{\delta}\right)\right)\right)\\
&=
O\left(\left(L + V\log\left(\frac{T}{\delta}\right)\right)\left(L + \log\left(\frac{T}{\delta}\right)\right) \left(\Dim(\cF, \cP)^{\frac{\alpha+1}{\alpha+2}}T^{\frac{1}{\alpha+2}} + \log\left(\frac{\log T}{\delta}\right)\right)\right)\\
&=
V\log^2\left(\frac{T}{\delta}\right)
O\left(
\Dim(\cF, \cP)^{\frac{\alpha+1}{\alpha+2}}T^{\frac{1}{\alpha+2}} + \log\left(\frac{\log T}{\delta}\right)
\right)~,
\end{align*}
where in the second equality we used the assumption that $\Dim(\cF, \cP)<T$.

If $2^{-L+2} > \epsilon_0$, then $2^L \leq \frac{4}{\epsilon_0}$. Plugging these bounds into Theorem \ref{thm: query bound for nonlinear model} and Theorem \ref{thm: regret bound for nonlinear model} gives
\begin{align*}
    N_T(\cP) 
    =&
    O\left(V
    \log^2
    \left(
    \frac{T}{\delta\epsilon_0}\right)\frac{\Dim(\cF, \cP)}{\epsilon_0^2}
    \right)
    =V\log^2
    \left(
    \frac{T}{\delta}\right)\left(1 + \log^2\left(\frac{1}{\epsilon_0}\right)\right)
    O\left(
    \frac{\Dim(\cF, \cP)}{\epsilon_0^2}
    \right)
    \\
    R_T(\cP) 
    =&
    O\left(V
    \log^2
    \left(
    \frac{T}{\delta\epsilon_0}\right)\frac{\Dim(\cF, \cP)}{\epsilon_0}\right)
    =V\log^2
    \left(
    \frac{T}{\delta}\right)\left(1 + \log^2\left(\frac{1}{\epsilon_0}\right)\right)
    O\left(
    \frac{\Dim(\cF, \cP)}{\epsilon_0}
    \right)
\end{align*}

Lastly, (\ref{ineq: high prob event on x}) and lemma \ref{lma: control of gap: nonlinear case} together yield
\[
\rP\left(\bar{\ep}\bigcap\left(\bigcap_{\ell=1}^L \ep_\ell\right)\right) \geq 1 - 2\delta~,
\]
which concludes the proof.
\end{proof}

We now turn the bound on the weighted cumulative regret $R_T(\cP)$ in the previous corollary into a bound on the excess risk. We can write
\begin{align*}
\cL(\widehat{f}) - \cL(f_\star) 
&= 
\rE_{(\x,y) \sim \cD}\Bigl[\ind{y \neq \sign(\widehat{f}(\x) - 1/2)} - \ind{y \neq \sign(f_\star(\x) - 1/2)}\Bigl]\\
&= 
\rE_{\x \sim \cD_{\cX}} \Bigl[\rE_{y \sim \cD_{\cY|\cX}}\bigl[\ind{y \neq \sign(\widehat{f}(\x) - 1/2)} - \ind{y \neq \sign(f_\star(\x) - 1/2)}\bigl]\Bigl]\\
&= 
\rE_{\x \sim \cD_{\cX}} \Bigl[\ind{\sign(\widehat{f}(\x) - 1/2) \neq \sign(f_\star(\x) - 1/2)}\,|f_\star(\x) - 1/2| \Bigl]~,
\end{align*}
where $\widehat{f}$ is the hypothesis returned by Algorithm \ref{a:batch_al_nonlinear}. Now, simply observe that
\[
\ind{\sign(\widehat{f}(\x) - 1/2) \neq \sign(f_\star(\x) - 1/2)}\,|f_\star(\x) - 1/2| 
\]
has the same form as the function $\phi(\widehat{f},\x)$ in Appendix \ref{as:ancillary} on which the uniform convergence result of Theorem \ref{thm:uniformconvergence} (with $\langle\hw,\x \rangle$ replaced by $\widehat{f}(\x)-1/2$) applies, with $\widehat{\epsilon}(\delta)$ therein replaced by the bound on $R_T(\cP)$ borrowed from Corollary \ref{thm: upper bounds, nonlinear case}.
This allows us to conclude that with probability at least $1-\delta$
\begin{align*}
\cL(\widehat{f}) - \cL(f_\star)
=&
V\log^2\left(\frac{T}{\delta}\right)\left(1 + \log^2\left(\frac{1}{\epsilon_0}\right)\right)O\left(
\mathrm{max}
\left\{
\left(\frac{\Dim(\cF, \cP)}{T}\right)^{\frac{\alpha+1}{\alpha+2}}
,~\frac{\Dim(\cF, \cP)}{T\epsilon_0}
\right\} 
\right)\\ 
&+
O\left(\frac{\log\left(\frac{\log T}{\delta}\right) + V \log\left(2eT\right)}{T}\right)~,
\end{align*}
as claimed in Theorem \ref{thm:main_nonlinear} in the main body of the paper.

\section{Ancillary technical results}\label{as:ancillary}
This section collects ancillary technical results that are used throughout the appendix.

We first recall the following version of the Hoeffding's bound.
\begin{lemma}\label{l:rademacher}
Let $a_1,\ldots,a_T$ be $T$ arbitrary real numbers, and $\{\sigma_1,\ldots,\sigma_T\}$ be $T$ i.i.d. Rademacher variables, each taking values $\pm 1$ with equal probability. Then for any $\epsilon \geq 0$
\[
\rP\left(\sum_{t=1}^T \sigma_t a_t \geq \epsilon \right) \leq \exp \left(-\frac{\epsilon^2}{2\sum_{t=1}^T a_t^2}\right)~,
\]
where the probability is with respect to $\{\sigma_1,\ldots,\sigma_T\}$. 
\end{lemma}

Let us consider the linear case first. 
Define the function $\phi\,:\, \rR^d\times\cP \rightarrow [0,1]$ as
\[
\phi(\hw,\x) = \ind{\sign\ang{\hw,\x}\neq\sign\ang{\w^*,\x}}\rho(\ang{\w^*,\x})~,
\]
where $\rho(\cdot)$ has range in $[0,1]$, and does not depend on $\hw$.
We have the following standard covering result, which is a direct consequence of Sauer-Shelah lemma (e.g., \cite{sa72}).
\begin{lemma}\label{lma:sauer}
Consider any given $S_T=\{\x_1,\ldots,\x_T\} \in \rR^d$, and let 
\[
\Phi(S_T) = \big|\{[\phi(\hw,\x_1),\ldots,\phi(\hw,\x_T)]\,:\, \hw \in \rR^d\}\big|~.
\]
We have, when $T \geq d$,
\[
\Phi(S_T) \leq \left(\frac{eT}{d}\right)^d~.
\]

\end{lemma}
The following result gives a bound on $L_\infty$ covering number for VC subgraph class. 
\begin{lemma}\label{lma: bound for covering number}
Let $f \in \cF$ be a $[0,1]$-valued function of $\x \in \cX$. 
Assume the binary function $\ind{z \leq f(\x)}$ defined on $\rR \times \cX$ has VC-dimension $V$, i.e. $\cF$ has VC-subgraph dimension $V$. Then given $S_T=\{\x_1,\ldots,\x_T\} \subset \cX$, with $T \geq V$,
\[
N(\epsilon, \cF, L_\infty(S_T)) \leq \left(\frac{eT}{V \epsilon}\right)^V .
\]
\end{lemma}
\begin{proof}
As $N(\cdot, \cF, L_\infty(S_T))$ is a decreasing function of $\epsilon$, we can assume that $1/\epsilon$ is an integer. Consider discretization of $[0,1]$ by $m=1/\epsilon$ points $U_1,\ldots,U_m$ so that $\forall U \in [0,1]$, $\min_i |U-U_i| \leq \epsilon$. Consider $mT$ points 
$\{(U_j,x_i):i,j\}$, by Sauer's lemma, the cardinality of  $\{\left(\ind{U_j \leq f(\x_i)}\right)_{i,j}\,:\, f \in \cF\}$ is at most $\left(\frac{eT}{V \epsilon}\right)^V$. Let $f_1,\ldots,f_N$ attain all distinctive values of $\{\left(\ind{U_j \leq f(\x_i)}\right)_{i,j}: f \in \cF\}$. 
Given any $f \in \cF$, there exists $f_k$ such that 
$\ind{U_j \leq f(\x_i)}= \ind{U_j \leq f_k(\x_i)}$ for all $i$ and $j$. It is easy to check that this implies that $|f(\x_i)-f_k(\x_i)| \leq \epsilon$ for all $\x_i \in S_T$. 
\end{proof}

The next result follows from a standard symmetrization argument. 
\begin{lemma}\label{lma:symmetrization}
Let $\cX = \rR^d$, $S_T = (\x_1,\ldots, \x_T)$ be a sample drawn i.i.d. according to $\cD_{\cX}$ and $S'_T = (\x'_1,\ldots, \x'_T)$ be another sample drawn according to $\cD_{\cX}$, with $T \geq d$. Then with probability at least $1-\delta$
\[
\sum_{t=1}^T \phi(\hw,\x'_t) \leq 3\sum_{t=1}^T \phi(\hw,\x_t) +  8 \log(1/\delta) + 8 d \log (2e T/d)~,
\]
uniformly over $\hw \in \rR^d$.
\end{lemma}
\begin{proof}
Let $\{\sigma_1,\ldots,\sigma_T\}$ be independent Rademacher variables as in Lemma \ref{l:rademacher}.
We can write, for any $\epsilon \geq 0$,\\
\begin{align*}
\rP&\left(\exists \hw \in \rR^d : \sum_{t=1}^T \phi(\hw,\x'_t) \geq 3\sum_{t=1}^T \phi(\hw,\x_t) + 2\epsilon\right)\\
&=
\rP\left(\exists \hw \in \rR^d : \sum_{t=1}^T \Bigl[\phi(\hw,\x'_t)-\phi(\hw,\x_t)\Bigl] \geq \frac{1}{2}\,\sum_{t=1}^T \Bigl[\phi(\hw,\x_t) +\phi(\hw,\x'_t)\Bigl] +\epsilon\right)\\
&=
\rP\left(\exists \hw \in \rR^d: \sum_{t=1}^T \sigma_t \Bigl[\phi(\hw,\x'_t)-\phi(\hw,\x_t)\Bigl] \geq \frac{1}{2}\,\sum_{t=1}^T \Bigl[\phi(\hw,\x_t) +\phi(\hw,\x'_t)\Bigl] + \epsilon\right) \\
&\leq
\rP\left(\exists \hw \in \rR^d: \sum_{t=1}^T \sigma_t \phi(\hw,\x'_t) \geq \frac{1}{2}\,\sum_{t=1}^T \phi(\hw,\x'_t) + \frac{\epsilon}{2}\right) \\
&\qquad + \rP\left(\exists \hw \in \rR^d : \sum_{t=1}^T -\sigma_t \phi(\hw,\x_t) \geq \frac{1}{2}\,\sum_{t=1}^T \phi(\hw,\x_t) + \frac{\epsilon}{2}\right) \\
&\leq 
2 \sup_{S_T} \rP\left(\exists \hw \in \rR^d : \sum_{t=1}^T \sigma_t \phi(\hw,\x_t) \geq \frac{1}{2}\,\sum_{t=1}^T \phi(\hw,\x_t) + \frac{1}{2}\,\epsilon\,\Bigl|\,S_T \right)\\
&\leq 
2 \left(\frac{eT}{d}\right)^d \sup_{S_T, \hw \in \rR^d} \rP\left(\sum_{t=1}^T \sigma_t \phi(\hw,\x_t) \geq \frac{1}{2}\,\sum_{t=1}^T \phi(\hw,\x_t) + \frac{1}{2}\,\epsilon\,\Bigl|\,S_T \right) \\
&{\mbox{(from the union bound and Lemma \ref{lma:sauer})}}\\
&\leq 
2 \left(\frac{eT}{d}\right)^d \sup_{S_T, \hw \in \rR^d} \exp\left(-\frac{(\sum_{t=1}^T \phi(\hw,\x_t) + \epsilon)^2}{8\sum_{t=1}^T \phi(\hw,\x_t)^2}\right) \\
&{\mbox{(from Lemma \ref{l:rademacher})}}\\
&\leq 
2 \left(\frac{eT}{d}\right)^d \exp(-\epsilon/4)~,
\end{align*}
the last inequality deriving from the fact that, since $\phi(\hw,\x_t) \in [0,1]$,
\[
\frac{(\sum_{t=1}^T \phi(\hw,\x_t) + \epsilon)^2}{\sum_{t=1}^T \phi(\hw,\x_t)^2} \geq \frac{(\sum_{t=1}^T \phi(\hw,\x_t) + \epsilon)^2}{\sum_{t=1}^T \phi(\hw,\x_t)} \geq 2\epsilon~.
\]
Take $\epsilon$ such that 
$\delta=2 \left(\frac{eT}{d}\right)^d \exp(-\epsilon/4)$, to obtain the claimed bound.
\end{proof}
%
\begin{theorem}\label{thm:uniformconvergence}
With the same notation and assumptions as in Lemma \ref{lma:symmetrization}, let $\hw \in \rR^d$ be a function of $S_T$ such that
\[
\frac{1}{T} \sum_{t=1}^T \phi(\hw,\x_t) \leq \widehat{\epsilon}(\delta) 
\]
holds with probability at least $1-\delta$, for some $\widehat{\epsilon}(\delta) \in [0,1]$. 
Then with probability at least $1-3\delta$:
\[
\rE_{\x \sim \cD}\phi(\hw,\x) \leq 4\widehat{\epsilon}(\delta) + \frac{22 \log\left(\frac{1}{\delta}\right) + 11 d \log\left(\frac{2eT}{d}\right)}{T}~.
\]
\end{theorem}
\begin{proof}
Use the multiplicative Chernoff bound
\[
\rE_{\x \sim \cD} \phi(\hw,\x) \leq \frac{4}{3T}\,\sum_{t=1}^T \phi(\hw,\x'_t) +  \frac{32}{3T}\, \log(1/\delta)~,
\]
and then apply Lemma \ref{lma:symmetrization} to further bound the right-hand side.
\end{proof}

To control noise terms, which are 1-subgaussian random variables, we provide the following lemma which is a direct implication of Chernoff bound. 
\begin{lemma}\label{lma: concentration of sub-gaussian r.v.}
Suppose $\xi$ is a $\sigma$-subgaussian random variable, then for any $\delta>0$,
\[
\rP\left(|\xi| \geq \sqrt{2\sigma^2\log(2/\delta)} \right) \leq \delta~.
\]
\end{lemma}

\begin{lemma}\label{lma: concentration of square of sub-gaussians}
Let $A=[a_{ij}]\in\mathbb{R}^{m\times n}$ be a matrix. Suppose $\xi_1,\ldots,\xi_n$ are independent $\sigma$-subgaussian random variables. Then for any $\delta>0$, 
\[
\rP\left(\|A\xi\|^2 >2\sigma^2\log\frac{2m}{\delta}\tr(A A^\top)\right) \leq \delta~,
\]
where $\xi=(\xi_1,\ldots,\xi_n)^\top$.
\end{lemma}
\begin{proof}
Consider
\[
(A\xi)_i = \sum_{j=1}^n a_{ij}\xi_j~,
\]
the $i$th component of vector $A\xi$.
Note that $(A\xi)_i$ is a $\sigma\sqrt{\sum_{j=1}^na_{ij}^2}$-subgaussian random variable, by lemma \ref{lma: concentration of sub-gaussian r.v.} we have
\[
\rP\left(\left|\sum_{j=1}^n a_{ij}\xi_j\right| \geq \sqrt{2\sigma^2\sum_{j=1}^{n}a_{ij}^2\log\frac{2m}{\delta}}\right) \leq \frac{\delta}{m}~.
\]
A union bound over $i$ gives, with probability at least $1-\delta$, 
\[
\left(\sum_{j=1}^n a_{ij}\xi_j\right)^2 \leq 2\sigma^2\sum_{j=1}^{n}a_{ij}^2\log\frac{2m}{\delta}~,
\]
uniformly over $i=1,\ldots,m$.
Therefore, with probability at least $1-\delta$,
\[
\|A\xi\|^2 =\sum_{i=1}^m \left(\sum_{j=1}^n a_{ij}\xi_j\right)^2 \leq 2\sigma^2\log\frac{2m}{\delta}\sum_{i, j}a_{ij}^2 = 2\sigma^2\log\frac{2m}{\delta}\tr(A A^\top)~,
\]
as claimed.
\end{proof}

Let us now consider the nonlinear case analyzed in Section \ref{sa:nonlinear}. Similar to the linear case, define the function $\phi\,:\, \cF\times\cP \rightarrow [0,1]$ as
\[
\phi(\widehat{f},\x) = \ind{\sign(\widehat{f}(\x) - 1/2) \neq \sign(f_\star(\x) - 1/2)}\,\rho(f_\star(\x))~,
\]
and $\rho(\cdot)$ has range in $[0, 1]$.

As for the counterparts to Theorem \ref{thm:uniformconvergence}, simply observe that
Lemmas \ref{lma:sauer} and \ref{lma:symmetrization} hold in the more general case with $d$ replaced by $V$, where $\cF$ is a function class having finite VC-subgraph dimension $V$.

In particular, given any $S_T=\{\x_1,\ldots,\x_T\} \in \cX$, define 
\[
\Phi(S_T) = \big|\{[\phi(\widehat{f},\x_1),\ldots,\phi(\widehat{f},\x_T)]\,:\, \widehat{f} \in \cF\}\big|~.
\]
We then have, when $T \geq V'$,
\[
\Phi(S_T) \leq \sum_{i=0}^{V'} \binom{T}{i}
\]
where $V'$ is the VC dimension of the binary-valued class
\[
\left\{\sign(\widehat{f}(\x) - 1/2)\,:\, \widehat{f} \in \cF \right\}~.
\]
Now it is easy to see that $V' \leq V$, so that we also have
\[
\Phi(S_T) \leq  \sum_{i=0}^{V} \binom{T}{i} \leq
\left(\frac{eT}{V}\right)^{V}~.
\]
%
With this modification, Theorem \ref{thm:uniformconvergence} holds with factor
$
d\log\left(\frac{2eT}{d}\right)
$
therein replaced by
$
V\log\left(\frac{2eT}{V}\right),
$
once we also replace $\langle\hw,\x \rangle$ by $\widehat{f}(\x)-1/2$.

\end{document}